\documentclass[11pt]{article}

\usepackage[preprint]{acl}

\usepackage{times}
\usepackage{latexsym}

\usepackage{amsmath}
\usepackage{amsthm}
\usepackage{amssymb}
\usepackage{thm-restate} % restatable theorems
\usepackage[T1]{fontenc}
\usepackage[utf8]{inputenc}

\usepackage{microtype}

\usepackage{inconsolata}

\usepackage{graphicx}
\usepackage{subcaption}
\usepackage{booktabs}

\newcommand{\update}[1]{\textcolor{black}{#1}}

\usepackage{graphicx}
\usepackage{amsmath,amssymb,amsthm,latexsym}
\usepackage{thm-restate} % restatable theorems
\usepackage{multirow}
\usepackage{hyperref}
\usepackage{stmaryrd}

\usepackage{booktabs}
\usepackage[normalem]{ulem}
\useunder{\uline}{\ul}{}
\usepackage{multicol}
\usepackage{ifthen}
\usepackage{thmtools}
\usepackage{tabularx}
\usepackage{xspace}
\usepackage{bbm}
\usepackage[normalem]{ulem}
\usepackage{hyperref}
\usepackage{float}
\usepackage{xcolor}[table]
\definecolor{ai2pink}{HTML}{f0529c}
\definecolor{ai2midpink}{HTML}{fad3e5}
\definecolor{ai2lightpink}{HTML}{fbecf3}   
\definecolor{ai2midwhite}{HTML}{f2e5d9}
\definecolor{ai2green}{HTML}{0fcb8c}
\definecolor{ai2lightgreen}{HTML}{e7f9f3}
\definecolor{ai2darkgreen}{HTML}{105257}
\definecolor{ai2purple}{HTML}{B932EB}
\definecolor{ai2lightpurple}{HTML}{f7e8fc}

\usepackage{rotating}

\newtheorem{example}{Example}

\usepackage[most]{tcolorbox}
\usepackage{wrapfig}
\usepackage{graphicx}
\usepackage{array}
\usepackage{listings}
\definecolor{githubblue}{RGB}{49,46,138}
\definecolor{schemegreen}{RGB}{15,137,15}
\definecolor{operator}{RGB}{0,.3,.7}
\definecolor{azure}{rgb}{0,.3,.7}
\definecolor{comment_color}{rgb}{0.24, 0.51, 0.51}
\definecolor{applegreen}{rgb}{0.55, 0.71, 0.0}

\newcommand{\modellog}{\textcolor{gray}{\textsc{ModelLog}}\xspace}
\newcommand{\probCT}{\textcolor{gray}{\textsc{ProbCT}}\xspace}

\definecolor{arsenic}{rgb}{0.23, 0.27, 0.29}
\definecolor{deepcarmine}{rgb}{0.66, 0.13, 0.24}

\DeclareRobustCommand{\propvar}{%
  \text{\textcolor{blue}{$\mathsf{P}$}}
}

\DeclareRobustCommand{\formula}{%
  \text{\textcolor{blue}{$\mathsf{F}$}}
}

\DeclareRobustCommand{\logicor}{%
  \text{\textcolor{deepcarmine}{$\lor$}}
}
\DeclareRobustCommand{\logicand}{%
  \text{\textcolor{deepcarmine}{$\land$}}
}
\DeclareRobustCommand{\implication}{%
  \text{\textcolor{deepcarmine}{$\to$}}
}
\DeclareRobustCommand{\biconditional}{%
  \text{\textcolor{deepcarmine}{$\leftrightarrow$}}
}

\DeclareRobustCommand{\biconditional}{%
  \text{\textcolor{deepcarmine}{$\leftrightarrow$}}
}
\tcbuselibrary{listings}%
\tcbset{listing engine={listings}}

\usepackage{colortbl}
\definecolor{halfgray}{gray}{0.55}
\definecolor{ipython_frame}{RGB}{207, 207, 207}
\definecolor{deepblue}{rgb}{0,0,0.5}
\definecolor{deepred}{rgb}{0.6,0,0}
\definecolor{deepgreen}{rgb}{0,0.5,0}

\lstdefinelanguage[]{iPython}[]{python}{
    commentstyle=\color{cyan}\ttfamily,
    stringstyle=\color{red},
    keywordstyle=\color{deepblue}\ttb,
    keepspaces=true,
    showspaces=false,
    showstringspaces=false,
    morekeywords=[4]{assert, Gen},
    morekeywords=[3]{Implies,Or,And,Not},
    keywordstyle=[3]\color{blue}\bf\ttfamily,
    keywordstyle=[4]\color{deepred},
    frame=l,
    numbers=left,
    numberstyle=\normalsize\color{halfgray},
    xleftmargin={0.2cm},
    basicstyle=\fontfamily{cmtt}\normalsize,
    keywordstyle=\color{deepgreen},
}
\lstdefinelanguage{Scheme}{
  morekeywords=[1]{XOR,M,T,Good,Implies,And,Or,Biconditional,Not,Ref,Mref,Noise},
  morekeywords=[2]{begin},
  morekeywords=[3]{import, export},
  alsodigit=!\$\%&*+-./:<=>?@^_~,
  sensitive=true,
  escapeinside=`',
  morecomment=[l]{;},
  morecomment=[l]{\#},
  morecomment=[s]{\#|}{|\#},
  morestring=[b]",
  basicstyle=\ttfamily,
  keywordstyle=\bf\ttfamily\color[rgb]{0,.3,.7},
  commentstyle={\color[rgb]{0.24, 0.51, 0.51}},
  stringstyle={\color[rgb]{0.75, 0.49, 0.07}},
  upquote=true,
  breaklines=false, %<-- problem for color
  breakatwhitespace=true,
  literate=*{`}{{`}}{1},
  showstringspaces=false,
}
\lstdefinestyle{churchstyle}{
  commentstyle=\color{gray},
  keywordstyle=\color{githubblue},
  numberstyle=\color{black}, % took off tiny 
  stringstyle=\color{red},
  basicstyle=\ttfamily\color{githubblue},
  breakatwhitespace=false,         
  breaklines=false,  %<--- problem for color              
  captionpos=b,                    
  keepspaces=true,                 
  numbers=none,                    
  numbersep=5pt,                  
  showspaces=false,                
  showstringspaces=false,
  showtabs=false,                  
  tabsize=2,
  literate=*{\{}{{\textcolor{NavyBlue}{\{}}}{1}
        {\}}{{\textcolor{black}{\}}}}{1}
        {[}{{\textcolor{black}{[}}}{1}
        {]}{{\textcolor{black}{]}}}{1}
        {(}{{\textcolor{black}{(}}}{1}
        {)}{{\textcolor{black}{)}}}{1}%
}
\lstnewenvironment{tabularlstlisting}[1][]
 {%
  \lstset{aboveskip=-1.3ex,belowskip=-2.5ex,#1}%
 }
 {}

\lstdefinelanguage[]{problang}[]{python}{
    commentstyle=\color{cyan}\ttfamily,
    stringstyle=\color{red},
    keywordstyle=\color{deepblue}\ttb,
    keepspaces=true,
    showspaces=false,
    showstringspaces=false,
    morekeywords=[4]{assert, Gen},
    morekeywords=[3]{Implies,Or,And,Not,ValueError, Decision, Flip, observe, Factor, Categorical, prod},
    keywordstyle=[3]\color{blue}\bf\ttfamily,
    keywordstyle=[4]\color{deepred},
    frame=l,
    numbers=left,
    numberstyle=\scriptsize\color{halfgray},
    xleftmargin={0.2cm},
    basicstyle=\fontfamily{cmtt}\scriptsize,
    keywordstyle=\color{deepgreen},
}
\usetikzlibrary{positioning, shapes, arrows.meta}
\usetikzlibrary{arrows.meta, decorations.pathreplacing}
\usetikzlibrary{calc}

\newcommand{\wmcplotrow}[2]{%
  \raisebox{0.5cm}{\rotatebox{90}{\tiny #1}}%
  \hspace{2pt}%
  \includegraphics[width=0.18\textwidth]{paper_plots/#1/#2/transitivity/wmc_score_distribution.pdf}%
  \includegraphics[width=0.18\textwidth]{paper_plots/#1/#2/forward_implication/wmc_score_distribution.pdf}%
  \includegraphics[width=0.18\textwidth]{paper_plots/#1/#2/negation_consistency/wmc_score_distribution.pdf}%
  \includegraphics[width=0.18\textwidth]{paper_plots/#1/#2/mutual_exclusivity/wmc_score_distribution.pdf}%
  \includegraphics[width=0.18\textwidth]{paper_plots/#1/#2/spatial_exclusivity/wmc_score_distribution.pdf}\\[2pt]%
}
\newcommand{\wmcplotrowlabeled}[2]{%
  \raisebox{0cm}{\rotatebox{90}{\tiny #1}}%
  \hspace{2pt}%
  \setlength{\tabcolsep}{0pt}%
  \begin{tabular}{@{}ccccc@{}}
  \includegraphics[width=0.19\textwidth]{paper_plots/#1/#2/transitivity/wmc_score_distribution.pdf}%
  &\includegraphics[width=0.18\textwidth]{paper_plots/#1/#2/forward_implication/wmc_score_distribution.pdf}%
  &\includegraphics[width=0.18\textwidth]{paper_plots/#1/#2/negation_consistency/wmc_score_distribution.pdf}%
  &\includegraphics[width=0.18\textwidth]{paper_plots/#1/#2/mutual_exclusivity/wmc_score_distribution.pdf}%
  &\includegraphics[width=0.18\textwidth]{paper_plots/#1/#2/spatial_exclusivity/wmc_score_distribution.pdf}\\
   \footnotesize Transitivity & \footnotesize Forward Implication & \footnotesize Negation Consistency & \footnotesize Mutual Exclusivity & \footnotesize Spatial Exclusivity
  \end{tabular}\\[2pt]%
}

\usepackage{booktabs}

\title{From Token Probabilities to Semantic Constraints: Towards Declarative Probabilistic Evaluation of Language Models}

\author{Kyle Richardson\textsuperscript{1}
Cullen Anderson\textsuperscript{2}
Pranav Balakrishnan\textsuperscript{2}
Takuto Ban\textsuperscript{2}
\textbf{Daksha Ladia}\textsuperscript{2} \\ 
\textbf{Ankita Gupta}\textsuperscript{2} \quad 
\textbf{Marisa Hudspeth}\textsuperscript{2} \\  
\textsuperscript{1}Allen Institute for AI \quad
\textsuperscript{2}University of Massachusetts Amherst \\
\texttt{kyler@allenai.org}, \,
\texttt{\{cyanderson,pranavbalakr,tban,dladia\}@umass.edu}  
}

\begin{document}
\maketitle

\begin{abstract}

\update{While Large Language Models have improved rapidly, many fundamental questions remain about how to evaluate the knowledge and reasoning abilities they acquire, and how such evaluations relate to the learning signals used in pre-training. In this paper, we propose \modellog, a declarative probabilistic framework for pre-training evaluation that makes the semantic structure of model behavior explicit and provides new formal tools for relating evaluation to learning. \modellog specifies evaluation targets as symbolic constraints over token-level predictions and measures how strongly a model’s distribution satisfies those constraints. We explore the framework through a new suite of tasks targeting negation, mutual exclusivity, and consistency, finding systematic failures that are difficult to characterize through token likelihood or answer accuracy alone. We further show that these evaluation scores can also be interpreted as losses, whose gradients reflect logical strength, informativeness, and variable-level sensitivity. This links evaluation and learning through a shared semantics, suggesting evaluation methods that diagnose model behavior while also helping to clarify the semantic structure of learning.}

\end{abstract}

\section{Introduction}

Modern pre-training relies on a remarkably simple learning signal: cross-entropy loss for next-token prediction. This objective has given rise to models with broad linguistic, factual, and reasoning capabilities, but at first glance it seems limited. While it provides positive, local supervision over observed continuations, many other richer semantic relations remain implicit. For example in Figure~\ref{fig:constraints}, next-token supervision might tell us that \emph{alcohol} is a sensible completion to the generic statement \textcolor{gray}{\emph{When driving, it is \underline{not} safe to drink \_\_}}, which expresses a broadly applicable factual or normative commitment. However, the objective does not directly state that \emph{alcohol} should be unlikely in the contradictory context \textcolor{gray}{\emph{When driving, it is safe to drink \_\_}}. Nor does it explicitly encode relations among multiple token predictions: paraphrases should preserve the same commitments, while contexts that commit to incompatible states of affairs should induce divergent predictions. This raises the natural question: \emph{why is cross-entropy such an effective signal for learning?} Relatedly, \emph{can richer learning signals be developed that improve the robustness of cross-entropy pre-training?}

\begin{figure}[t]
\centering
\begin{tikzpicture}[scale=0.72, transform shape,
    tokenbadge/.style={font=\tiny\bfseries, text=white, 
        rounded corners=1pt, inner sep=1.5pt, minimum size=0.3cm}]

% ======== Left: LM Predictions ========
\node[font=\footnotesize\bfseries, anchor=north, align=left] (Alabel) at (2,0) {(A) \textcolor{violet}{LM Token Predictions} \quad\quad\quad};
\node[below=-0.2cm of Alabel.south, anchor=north, font=\scriptsize, text=gray, text width=5.4cm, align=left] (Asub) {
    Example local next-token \colorbox{red!50}{\textcolor{white}{decisions}} for a language model $t \sim \mathbb{P}_{\mathcal{M}_{\theta}}(\cdot)$.
};

\node[below=-0.05cm of Asub.south west, anchor=north west,
      draw=gray!40, fill=white, rounded corners=2pt,
      inner sep=3pt, text width=4.8cm, font=\footnotesize] (sa) {
    When driving, it is not safe to drink \tikz[baseline]{\node[fill=red!15, draw=red!50, rounded corners=1.5pt, inner sep=2pt, anchor=base, thick](toka){\strut\textbf{alcohol}}; \node[tokenbadge, fill=red!50, anchor=west] at (toka.east) {a};}
};

\node[below=0.06cm of sa.south west, anchor=north west,
      draw=gray!40, fill=white, rounded corners=2pt,
      inner sep=3pt, text width=4.8cm, font=\footnotesize] (sb) {
    It is not true that while driving it is safe to drink \tikz[baseline]{\node[fill=red!15, draw=red!50, rounded corners=1.5pt, inner sep=2pt, anchor=base, thick](tokb){\strut\textbf{alcohol}}; \node[tokenbadge, fill=red!50, anchor=west] at (tokb.east) {b};}
};

\node[below=0.15cm of sb.south west, anchor=north west,
      draw=gray!40, fill=white, rounded corners=2pt,
      inner sep=3pt, text width=4.8cm, font=\footnotesize] (sc) {
    When driving, it is safe to drink \tikz[baseline]{\node[fill=blue!12, draw=blue!50, rounded corners=1.5pt, inner sep=2pt, anchor=base, thick](tokc){\strut\textbf{alcohol}}; \node[tokenbadge, fill=blue!50, anchor=west] at (tokc.east) {c};} \\
};

\node[below=0.06cm of sc.south west, anchor=north west,
      draw=gray!40, fill=white, rounded corners=2pt,
      inner sep=3pt, text width=4.8cm, font=\footnotesize] (sd) {
    It is true that while driving it is safe to drink \tikz[baseline]{\node[fill=blue!12, draw=blue!50, rounded corners=1.5pt, inner sep=2pt, anchor=base, thick](tokd){\strut\textbf{alcohol}}; \node[tokenbadge, fill=blue!50, anchor=west] at (tokd.east) {d};}
};

% ======== Query box under predictions ========
\node[below=0.2cm of sd.south west, anchor=north west,
      draw=violet!50, fill=violet!3, rounded corners=3pt,
      inner sep=4pt, text width=4.8cm, font=\footnotesize] (queries) {
    {\footnotesize \bfseries\textcolor{black}{(C) \textcolor{violet}{Queries over formulas $\mathcal{T}$}}}\\[1pt]
    \textcolor{gray}{\scriptsize Evaluation as structured probabilistic queries over constraints $\mathcal{T}_{i}$ weighted by $\mathbb{P}_{\mathcal{M}_{\theta}}$} \\[1pt]
    \textbf{Soft satisfaction:}
    \hspace{2pt}$\footnotesize \mathbb{P}(\mathcal{T}_{i}; \mathcal{M}_{\theta})$ $\footnotesize $ %\\[2pt]
    \textbf{Marginals:} %\\[1pt]
    \hspace{2pt}$\footnotesize \mathbb{P}($\colorbox{red!50}{\footnotesize\bfseries\color{white}$v$}$\, \footnotesize \mid \mathcal{T}_{i}; \mathcal{M}_{\theta})$ \\[2pt]
    \textbf{Gradients :} %\\[1pt]
    \hspace{2pt}$\footnotesize\nabla_{\theta}\, \mathbb{P}(\mathcal{T}_{i}; \mathcal{M}_{\theta})$ 
};

% ======== Right: Declarative Theory ========
\node[font=\footnotesize\bfseries, anchor=north, align=left] (Blabel) at (7.8,0) {(B) \textcolor{violet}{Declarative Constraints $\mathcal{T}$}};
\node[below=-0.1cm of Blabel.south, anchor=north, font=\scriptsize, text=gray, text width=4cm, align=left] (Bsub) {
    Boolean formulas specifying ideal token decisions for any model $\mathcal{M}(\cdot)$.
};

\node[below=-0.0cm of Bsub.south west, anchor=north west,
      draw=gray!50, fill=gray!2, rounded corners=3pt,
      inner sep=4pt, text width=4.2cm, font=\scriptsize, fill=gray!5] (theory) {
    {\footnotesize\bfseries Likely tokens $\mathcal{T}_{p}$}\\[1pt]
    \textcolor{gray}{Models $\mathcal{M}$ should prefer local factual token decisions.} \\[1pt]
    $\footnotesize \mathcal{M}($\colorbox{red!50}{\tiny\bfseries\color{white}a}$\scriptstyle)$
    \, \logicand \, 
    $\footnotesize \mathcal{M}($\colorbox{red!50}{\tiny\bfseries\color{white}b}$\scriptstyle)$
    \\[3pt]
    \tikz{\draw[gray!30] (0,0) -- (3.7cm,0);}\\[2pt]
    {\footnotesize\bfseries Unlikely tokens $\mathcal{T}_{n}$}\\[1pt]
    \textcolor{gray}{Models $\mathcal{M}$ should disprefer non-factual local tokens decisions.} \\[1pt]
    $\footnotesize\neg \mathcal{M}($\colorbox{blue!50}{\scriptsize\bfseries\color{white}c}$\scriptstyle)$
    \, \logicand \, 
    $\footnotesize\neg \mathcal{M}($\colorbox{blue!50}{\scriptsize\bfseries\color{white}d}$\scriptstyle)$
    \\[3pt]
    \tikz{\draw[gray!30] (0,0) -- (3.7cm,0);}\\[2pt]
    {\footnotesize\bfseries Paraphrase Symmetries $\mathcal{T}_{s}$}\\[1pt]
    \textcolor{gray}{Token decisions involving paraphrases should be compatible}\\[1pt]
    $\mathcal{M}($\colorbox{red!50}{\scriptsize\bfseries\color{white}a}$)  \biconditional \mathcal{M}($\colorbox{red!50}{\scriptsize\bfseries\color{white}b}$) \, \logicand \, $ \\[1pt]
    $\quad \mathcal{M}($\colorbox{blue!50}{\scriptsize\bfseries\color{white}c}$) \biconditional \mathcal{M}($\colorbox{blue!50}{\scriptsize\bfseries\color{white}d}$)$
    % $\footnotesize \mathcal{M}($\colorbox{red!50}{\scriptsize\bfseries\color{white}a}$\scriptstyle) 
    % \footnotesize \biconditional
    % \footnotesize\mathcal{M}($\colorbox{red!50}{\scriptsize\bfseries\color{white}b}$\scriptstyle)$\\[1pt]
    % $\footnotesize \mathcal{M}($\colorbox{blue!50}{\scriptsize\bfseries\color{white}c}$\scriptstyle) 
    % \biconditional
    % \mathcal{M}($\colorbox{blue!50}{\scriptsize\bfseries\color{white}d}$\scriptstyle)$
    \\[3pt]
    \tikz{\draw[gray!30] (0,0) -- (3.7cm,0);}\\[2pt]
    {\footnotesize\bfseries Incompatibilities $\mathcal{T}_{m}$}\\[1pt]
    \textcolor{gray}{Contradictory token predictions should be distinct and mutually exclusive.} \\[1pt]
    $(\mathcal{M}(\colorbox{red!50}{\scriptsize\bfseries\color{white}a}) \, \logicand \, \neg\mathcal{M}(\colorbox{blue!50}{\scriptsize\bfseries\color{white}b})) \, \logicor \,$ \\[1pt]
    $\quad (\neg\mathcal{M}(\colorbox{red!50}{\scriptsize\bfseries\color{white}a}) \, \logicand \, \mathcal{M}(\colorbox{blue!50}{\scriptsize\bfseries\color{white}b}))$ 
    % $\mathcal{M}($$)
    % \logicand
    % \neg \mathcal{M}($\colorbox{blue!50}{\scriptsize\bfseries\color{white}c}$\scriptstyle)$\\[1pt]
    % $\footnotesize \, \logicor\,
    % \neg \mathcal{M}($\colorbox{red!50}{\scriptsize\bfseries\color{white}a}$\scriptstyle) 
    % \, \logicand \,  
    % \footnotesize\mathcal{M}($\colorbox{blue!50}{\scriptsize\bfseries\color{white}c}$\scriptstyle)$
};

% ======== Arrows: predictions → theory ========
\draw[-{Stealth[length=3pt]}, red!40, thick, rounded corners=3pt, shorten >=2pt, shorten <=2pt]
    ([xshift=-0.1cm]sa.east) -- ([yshift=-0.53cm]theory.north west);
\draw[-{Stealth[length=3pt]}, red!40, thick, rounded corners=3pt, shorten >=2pt, shorten <=2pt]
    ([xshift=-0.1cm]sb.east) -- ([xshift=0.6cm]sb.east) |- ([yshift=-0.7cm]theory.north west);

\draw[-{Stealth[length=3pt]}, blue!40, thick, rounded corners=3pt, shorten >=2pt, shorten <=2pt]
    ([xshift=-0.1cm]sc.east) -- ([xshift=0.4cm]sc.east) |- ([yshift=-0.9cm]theory.north west);
\draw[-{Stealth[length=3pt]}, blue!40, thick, rounded corners=3pt, shorten >=2pt, shorten <=2pt]
    ([xshift=-0.1cm]sd.east) -- ([xshift=0.6cm]sd.east) |- ([yshift=-1.1cm]theory.north west);

% Arrow: queries → theory (points TO the theory)
\draw[-{Stealth[length=3pt]}, violet!50, thick, shorten >=2pt, shorten <=2pt]
    ([xshift=-0.10cm]queries.east) -- ([yshift=-2.25cm]theory.west)
    node[midway, above, sloped, text=violet!70] {\tiny $\mathbb{P}(\mathcal{M}(\cdot)) \sim$}
    node[midway, below, sloped, font=\tiny\itshape, text=violet!70] {$\mathbb{P}_{\mathcal{M}_{\theta}}(\cdot)$};
\end{tikzpicture}

\caption{An illustration of our evaluation framework \modellog that maps local next-token predictions from a pre-trained language model (\textbf{A}) into propositional variables $\mathcal{M}(\cdot)$ denoting prediction events. Symbolic constraints over these variables (\textbf{B}) specify the local semantic relations that should hold among them. Evaluation then becomes probabilistic inference over constraints (\textbf{C}): model token probabilities $\mathbb{P}_{\mathcal{M}_{\theta}}(\cdot)$ provide weights over assignments $\mathbb{P}(\mathcal{M}(\cdot))$, \textbf{satisfaction} or \textbf{marginal} probability measures how strongly the model supports the desired semantic structure, and the \textbf{gradients} of such inferences provide a direct link to learning. 
}
\label{fig:constraints}
\end{figure}
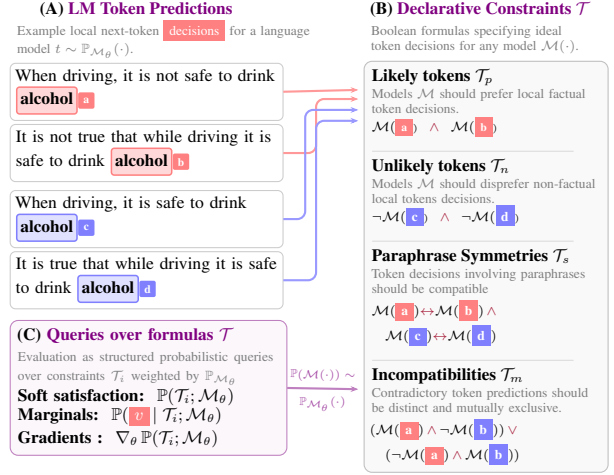

In this paper, we approach these broad questions \update{\emph{solely}} through the lens of pre-training evaluation, focusing in particular on how to \update{formalize} and precisely measure the degree to which \update{pre-trained} models satisfy semantic constraints. As illustrated in Figure~\ref{fig:constraints}, we develop a framework called \modellog where we express the target of an evaluation as an explicit \textbf{declarative theory}: a set of logical formulas specifying the relations that should ideally hold among a model’s local token-level predictions. For example, to express that \emph{alcohol} is compatible with both the prefix \textcolor{gray}{\emph{When driving, it is not safe to drink \_\_}} and its paraphrased form \textcolor{gray}{\emph{It is not true that when driving it is safe to drink \_\_}}, we use \update{logical variables} such as $\mathcal{M}(\colorbox{red!50}{\textcolor{white}{a}})$ and $\mathcal{M}(\colorbox{red!50}{\textcolor{white}{b}})$ to denote the corresponding prediction events. The paraphrase relation between these contexts can then be encoded by the constraint $\mathcal{M}(\colorbox{red!50}{\textcolor{white}{a}}) \biconditional \mathcal{M}(\colorbox{red!50}{\textcolor{white}{b}})$, requiring that the two prediction events agree. In this way, declarative theories let us move from isolated token likelihoods to explicit statements about the semantic relations that should hold among them.

Questions about model behavior can then be formulated as structured queries over the theories themselves. We focus on \textbf{probabilistic queries}: given a model’s token probabilities, we ask how much probability mass the model assigns to interpretations that \update{satisfy the constraints of a given theory}. \update{This builds on work in statistical relational learning \citep{getoor2007introduction, de2016statistical} and neuro-symbolic modeling \citep{MARRA2024104062, feldstein2024mapping}}, where logical \update{constraints} are combined with probabilistic weights to reason about structured events. In this way, probabilistic inference moves evaluation beyond binary correctness, measuring how strongly a model’s distribution supports \update{a target} semantic structure.

\update{The main goal of this paper is \update{to give a technical outline of} \modellog and \update{show how such a framework can both} clarify the semantics of evaluation and connect such semantics directly to learning. As such, the paper is primarily methodological in nature and offers a mix of formal and empirical results. On the formal side, we identify two interpretability challenges that arise when doing semantic reasoning over token probabilities. First, raw next-token probabilities measure how probability mass is divided among competing continuations, making them difficult to interpret directly as probabilities of semantic validity. We therefore rescale them using information about the local prediction distribution, treating semantic event probability as membership in a model’s effective ``live’’ region of plausible continuations \citep{holtzman2019curious, hewitt2022truncation}.  Second, raw constraint probabilities can overstate model competence, since some formulas are easy to satisfy by chance. We therefore introduce a notion of \textbf{constraint informativeness}, which accounts for chance satisfaction and is incorporated into our main probabilistic evaluation metric.}

\update{Since semantic satisfaction is differentiable in model probabilities, our evaluation scores can also be interpreted as losses. Surprisingly, we show how constraint informativeness, introduced as a measurement for evaluation, appears directly in the corresponding gradients for learning. This offers a formal perspective on why likelihood-style next-token objectives -- which we show in \modellog can be represented as maximally informative conjunctive formulas -- can potentially induce strong learning signals and provides a basis for systematically comparing them to other candidate objectives.}

\update{On the empirical side, we introduce \probCT (\textbf{Prob}abilistic \textbf{C}onsistency \textbf{T}ests) \update{a set of diagnostic tests inspired by prior work on consistency probing} \citep{elazar2021measuring, kassner2020negated}. We find that current pre-trained models perform poorly across these tasks, suggesting that pre-training does not reliably induce the structured semantic relations captured by our constraints. More importantly, \probCT illustrates how declarative theories and probabilistic queries can expose failures that are difficult to see from token likelihoods or accuracy alone. This includes an intriguing inverse-scaling \citep{mckenzie2023inverse} pattern we observe, where larger models systematically satisfy certain constraints less often than smaller ones.}

\paragraph{Contributions} \update{In line with the special track on \emph{New Missions for NLP Research}, we propose \modellog, a new evaluation framework that connects language model evaluation with techniques from statistical relational learning and probabilistic logic. Our main contribution is methodological: we show how declarative probabilistic evaluation can help to bring more semantic clarity to evaluation by treating evaluation data as compositional semantic objects. We illustrate this methodology through a mixture of formal results and empirical case studies on a new diagnostic benchmark called \probCT.}

\section{Related Work}

\paragraph{Language model evaluation} Our work connects to loss-based language model evaluation, which measures models using quantities such as perplexity \citep{jelinek1980interpolated, goodman2001bit, bengio2003neural, brown2020language, lozhkovsmollm2, groeneveld2024olmo, magnusson2024paloma}, and post-hoc behavioral evaluation, which tests models on diagnostic tasks and benchmarks \citep[\emph{inter alia}]{srivastava2023beyond, wang2018glue, petroni2019language, richardson2020probing, ribeiro2020beyond, liang2022holistic}. \update{We specifically take inspiration from behavioral studies that use consistency as a core diagnostic of model behavior \citep{elazar2021measuring, kassner2020negated, liu2024aligning, jang2022becel, novikova2025consistency}}. \modellog sits between these traditions: it uses model probabilities while targeting interpretable semantic capabilities. Unlike either, it evaluates rich declarative constraints over multiple predictions, yielding probabilistic, differentiable measurements that directly connect the semantics of evaluation with learning.

\begin{figure*}[t]
\centering
\begin{tikzpicture}[scale=0.82, transform shape]

% ======== Left: Token predictions ========
\node[font=\normalsize\bfseries, anchor=north west] (clabel) at (0,0) {(A) Token predictions};

\node[below=0.1cm of clabel.south west, anchor=north west,
      draw=gray!40, fill=gray!3, rounded corners=2pt,
      inner sep=5pt, text width=6.3cm, font=\footnotesize, tokenbadge/.style={font=\tiny\bfseries, text=white, 
        rounded corners=1pt, inner sep=1.5pt, minimum size=0.3cm}] (exbox) {
    When driving, it is not safe to drink \tikz[baseline]{\node[fill=red!15, draw=red!50, rounded corners=1.5pt, inner sep=2pt, anchor=base, thick](toka){\strut\textbf{alcohol}}; \node[tokenbadge, fill=red!50, anchor=west] at (toka.east) {a};} \\ 
     When driving, it is safe to drink \tikz[baseline]{\node[fill=blue!12, draw=blue!50, rounded corners=1.5pt, inner sep=2pt, anchor=base, thick](tokc){\strut\textbf{alcohol}}; \node[tokenbadge, fill=blue!50, anchor=west] at (tokc.east) {c};}
};

% ======== Left: Constraint ========
\node[below=0.1cm of exbox.south west, font=\normalsize\bfseries, anchor=north west] (constraint) {(B) Constraint $\formula$};

\node[below=0.1cm of constraint.south west, anchor=north west,
      draw=gray!40, fill=gray!3, rounded corners=2pt,
      inner sep=5pt, text width=6cm, font=\scriptsize] (cbox) {
    \textcolor{gray}{`Alcohol' is a valid token completion in one sentence, but not in both sentences.} \\[.2cm]
    $\formula = \big(\mathcal{M}($\colorbox{red!50}{\tiny\bfseries\color{white}a}$) \,\logicand\, \neg\mathcal{M}($\colorbox{blue!50}{\tiny\bfseries\color{white}c}$)\big)$
    $\;\logicor\; \big(\neg\mathcal{M}($\colorbox{red!50}{\tiny\bfseries\color{white}a}$) \,\logicand\, \mathcal{M}($\colorbox{blue!50}{\tiny\bfseries\color{white}c}$)\big)$
};

% ======== Right: Truth Table ========
\node[font=\normalsize\bfseries, anchor=north west] (tlabel) at (7.5,0.3) {(C) Boolean semantics};
\node[font=\normalsize\bfseries, anchor=north west] (llabel) at (7.5,-3.1 ) {(D) Language model};
\node[font=\normalsize\bfseries, anchor=north west] (plabel) at (14.5,-3.1 ) {(E) Uninformed Prior};
% ======== LM Box ========
\node[below=.1cm of llabel, anchor=north,
      draw=blue!70, fill=blue!60, rounded corners=4pt,
      inner sep=6pt, font=\large\bfseries, text=white,
      minimum width=1.8cm, minimum height=1cm] (lmbox) {$\mathcal{M}_{\theta}$};

\node[right=6.2cm of lmbox, anchor=east,
      draw=blue!70, fill=blue!60, rounded corners=4pt,
      inner sep=6pt, font=\large\bfseries, text=white,
      minimum width=1.8cm, minimum height=1cm] (baseline) {$\mathcal{M}_{0}$};

\node[below=0.0cm of tlabel.south west, anchor=north west] (table) {
\small
\begin{tabular}{
    >{\centering\arraybackslash}p{0.2cm}  
    >{\centering\arraybackslash}p{0.7cm} 
    >{\centering\arraybackslash}p{0.7cm} 
    >{\centering\arraybackslash}p{1.1cm}
    p{4.8cm}
    >{\centering\arraybackslash}p{2.4cm}
}
\toprule
& $\mathcal{M}(\colorbox{red!50}{\scriptsize\bfseries\color{white}a})$ 
& $\mathcal{M}(\colorbox{blue!50}{\scriptsize\bfseries\color{white}c})$ 
& $I \in \mathsf{I}(\formula)$ & \multicolumn{1}{c}{$S(I_{j}; \theta)$} & uniform \\
\midrule
$I_1$ & T & T & & $\mathbb{P}_{\theta}(\mathcal{M}(\colorbox{red!50}{\scriptsize\bfseries\color{white}a})) \cdot \mathbb{P}_{\theta}(\mathcal{M}(\colorbox{blue!50}{\scriptsize\bfseries\color{white}c}))$ & 0.25 \\[3pt]
\rowcolor{green!8}
$I_2$ & T & F & \checkmark & $\mathbb{P}_{\theta}(\mathcal{M}(\colorbox{red!50}{\scriptsize\bfseries\color{white}a})) \cdot (1  - \mathbb{P}_{\theta}(\mathcal{M}(\colorbox{blue!50}{\scriptsize\bfseries\color{white}c})))$ & 0.25 \\[3pt]
\rowcolor{green!8}
$I_3$ & F & T & \checkmark & $(1 - \mathbb{P}_{\theta}(\mathcal{M}(\colorbox{red!50}{\scriptsize\bfseries\color{white}a}))) \cdot \mathbb{P}_{\theta}(\mathcal{M}(\colorbox{blue!50}{\scriptsize\bfseries\color{white}c}))$ & 0.25 \\[3pt]
$I_4$ & F & F & & $(1 - \mathbb{P}_{\theta}(\mathcal{M}(\colorbox{red!50}{\scriptsize\bfseries\color{white}a}))) \cdot (1  - \mathbb{P}_{\theta}(\mathcal{M}(\colorbox{blue!50}{\scriptsize\bfseries\color{white}c})))$ & 0.25 \\ \bottomrule
\\[.3cm]
& & \multicolumn{3}{l}{$\qquad\,\, \text{WMC}(\formula; \theta) = S(I_{2}) + S(I_{3})$} & $\text{WMC}_{0}(\formula) = 0.5$
\end{tabular}
};

% constraints to semantics 
\draw[->, gray!70, very thick, rounded corners=4pt, shorten >=3pt, shorten <=3pt]
    ([yshift=-0.05cm]cbox.north) -- ++(0,0.0) |- ([xshift=0.2cm, yshift=-0.2cm]table.west);
% ======== Arrows: predictions + constraint → LM ========
\draw[-, black!70, very thick, rounded corners=4pt]
    ([yshift=0cm]exbox.east) -- ++(0.2,-0.0) |- ([xshift=-1.5cm, yshift=-0.75cm]lmbox.north west);
\draw[<-, black!70, very thick, rounded corners=4pt, shorten >=3pt, shorten <=3pt]
    ([xshift=0.12cm,yshift=-.25cm]lmbox.west) -- ++(-1.8,-0.0) -| ([yshift=-1cm]cbox);
%% prob 
\draw[->, black!70, very thick, rounded corners=4pt]
    ([xshift=0.0cm, yshift=-0.22cm]lmbox.east) -- ++(-0.05,-0.0) -| ([xshift=-2.8cm, yshift=-1cm]table)
    node[below right=1.4cm and -0.8cm] {\scriptsize $\hat{\mathbb{P}}_{\theta}(\mathcal{M}(x,\colorbox{red!50}{\textcolor{white}{$w_{j}$}})) \sim \mathbb{P}_{\mathcal{M}_{\theta}}(\colorbox{red!50}{\textcolor{white}{$w_{j}$}} \mid x_{<\, j})\qquad\qquad\,\, $}
    ;
\draw[->, black!70, very thick, rounded corners=4pt]
    ([xshift=0.15cm,yshift=-0.22cm]baseline.east) -- ++(0.01,0.0) -| ([xshift=6cm, yshift=0cm]table)
    node[below=2.5cm] {\scriptsize $\mathbb{P}_{0}(\mathcal{M}(x,\colorbox{red!50}{\textcolor{white}{$w_{j}$}})) = 0.5\hspace{2.8cm}$}
    ;
\end{tikzpicture}

\caption{Given \textbf{token predictions} (\textbf{A}) and a \textbf{constraint} $\formula$ over those predictions (\textbf{B}), evaluation reduces to weighted model counting (\textbf{WMC}) over interpretations $I_i$ of $F$ (\textbf{C}), weighted by model-induced probabilities from $\mathcal{M}_{\theta}$ (\textbf{D}). To make satisfaction scores interpretable, we also compare against an \textbf{uninformed prior} (\textbf{E}), which performs weighted model counting under uniform weights, yielding $\textbf{WMC}_0(\formula)$.}
\label{fig:wmc}
\end{figure*}
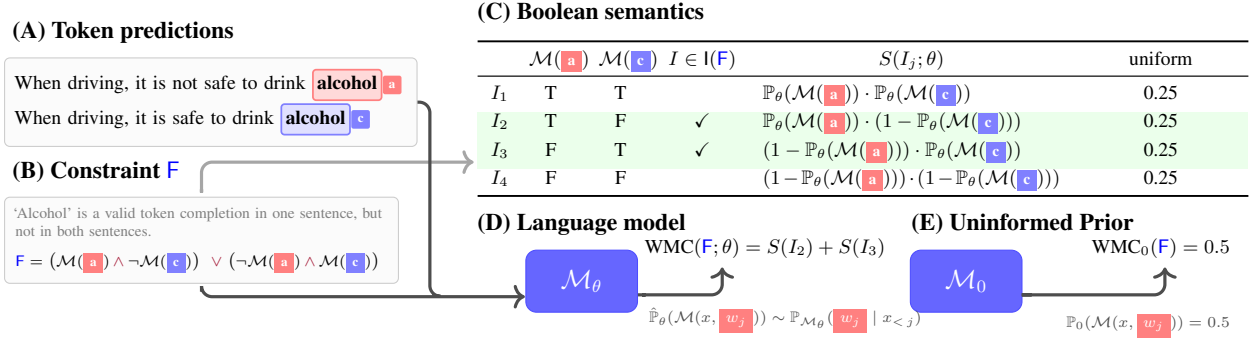

\paragraph{Neuro-symbolic modeling} We employ techniques from statistical relational learning that combine logical constraints with probabilistic weights \citep{de2015probabilistic, manhaeve2018deepproblog, li2019logic, li2023scallop} and exact inference techniques \citep{chavira2008probabilistic, fierens2015inference}. Similar techniques are also used in constraint-based neuro-symbolic learning methods, notably semantic loss \citep{xu2018semantic, ahmed2023pseudo, ahmed2024semantic, richardson2024understanding, calanzone2025logically} and semantic probabilistic layers \citep{ahmed2022semantic}, \update{as well as other approaches that couple structured probabilistic inference with LLMs \citep{dohan2022language, kassner2023language, lew2023sequential, cheng2026analytica, garg2026probabilistic, richardsoncots}}. \update{Our focus differs: rather than using constraints primarily as training objectives or output-layer structure for inference, we use them for LM evaluation and formal analysis.} 

\section{The \modellog Framework}

In this section, we define the core concepts and notation underlying the evaluation framework \modellog illustrated in Figure~\ref{fig:wmc}, starting with the five principles outlined below. In \S~\ref{sec:main_issues}, we then turn to the core technical obstacles that arise when applying our approach to pre-trained language model evaluation: \update{\emph{token probability calibration}, \emph{constraint probability rescaling} and \emph{selection}}.

\paragraph{1. Token prediction events are symbolic objects.}
Let $x = w_{1}, \ldots, w_{n}$ be a token sequence and \update{$\mathbb{P}_{\mathcal{M}_{\theta}}(\colorbox{red!50}{\textcolor{white}{$w_{j}$}} \mid x_{<\, j})$ be the next token probability assigned by an autoregressive model} $\mathcal{M}_{\theta}$. In \modellog, local model predictions are treated as symbolic propositions. We write $\mathcal{M}(x, \colorbox{red!50}{\textcolor{white}{$w_{j}$}})$ to denote the proposition that for model $\mathcal{M}$, the token $w_{j}$ is a valid completion at position $j$ in the context $x_{<j}$. Intuitively, this proposition abstracts away from the raw token probability and asks if the model supports $w_j$ as a plausible continuation. 

\paragraph{2. Relations between predictions are formulas.}
Semantic evaluation requires reasoning not only about individual predictions, but about relations among predictions. In \modellog, these relations are expressed as Boolean formulas $\formula$ over prediction variables $\mathcal{M}(\cdot)$ (denoted below in short form as $\propvar$), using standard logical operators such as $\logicand$, $\logicor$, $\implication$, $\biconditional$ and Booleans $\top/\bot$ (true/false).  Thus, a formula $\formula$ specifies the semantic structure that should hold among a set of local token predictions.

\paragraph{3. Formulas are weighted by model probabilities.}
To evaluate a formula probabilistically, we assign each prediction variable a weight derived from the language model. For a prediction event $\mathcal{M}(x,\colorbox{red!50}{\textcolor{white}{$w_{j}$}})$, this weight is based on the language model probability $\mathbb{P}_{\mathcal{M}_{\theta}}(\colorbox{red!50}{\textcolor{white}{$w_{j}$}} \mid x_{<j})$. Since \modellog represents prediction events as logical atoms, their semantics are naturally Bernoulli: in any interpretation, the event either holds or does not hold (we examine this closely in \S\ref{sec:main_issues}). We write the resulting event probability as $\mathbb{P}_{\theta}(\mathcal{M}(x,\colorbox{red!50}{\textcolor{white}{$w_{j}$}}))$. In this way, a language model induces a probability distribution over truth assignments. % to the symbolic prediction variables.

\paragraph{4. Evaluation is probabilistic satisfaction.}
Given a formula $\formula$ and probabilities for its variables, \modellog evaluates the degree to which the model satisfies the formula by computing $\mathbb{P}_{\theta}(\formula)$, the probability that $\formula$ holds under the model-induced distribution over truth assignments. Equivalently, this is the total probability mass assigned to interpretations $I$ that satisfy the declarative theory $\formula$. We compute this quantity using \textbf{weighted model counting} \citep{chavira2008probabilistic}:
\begin{align}
\mathbb{P}_{\theta}(\formula)
&= \text{WMC}(\formula; \theta) \\ 
&:= \sum_{I \in \mathsf{I}(\formula)}
\underbrace{\prod_{\propvar : I(\propvar)=\text{T}} \hspace{-.3cm} \mathbb{P}_{\theta}(\propvar) \cdot
\prod_{\propvar : I(\propvar)=\text{F}} \hspace{-.3cm} \left(1 - \mathbb{P}_{\theta}(\propvar)\right)}_{\text{probability } S(I;\theta) \text{ of interpretation } I} \nonumber
\end{align}
$\mathbb{P}_{\theta}(\formula)$ is then our main evaluation score: a graded measure of how strongly the model’s probabilities support the desired semantic constraint.

\paragraph{5. Learning is tied to satisfaction gradients.}
Since satisfaction probabilities are differentiable functions of the model-induced variable probabilities, \modellog exposes a natural connection \update{to}  learning. For a formula $\formula$, we can define the corresponding \textbf{semantic loss} from \citet{xu2018semantic}: 
\begin{align}
\ell_{\mathrm{sl}}(\formula; \theta) := - \log \mathbb{P}_{\theta}(\formula).
\label{eq:sl}
\end{align}
This loss penalizes the model when it assigns low probability mass to interpretations satisfying the declarative theory. Thus, the same quantity used for evaluation can also be differentiated to ask how changes in local prediction probabilities would affect semantic satisfaction. Later, we use this connection to show that semantic properties of the constraints that are relevant for evaluation appear in the gradients of the semantic loss $\nabla \ell_{\text{sl}}(\formula; \theta)$.

\begin{example}[\textcolor{black}{\textcolor{blue}{Likelihood formula}}]
%\paragraph{Likelihood and cross-entropy as a special case.}
An important special case is standard language-model likelihood. Given a token sequence $x = w_{1}, \ldots, w_{n}$, define the \textbf{likelihood formula} below as the conjunction of the observed token-prediction events $\formula_{\ell}(x)$.
\begin{lstlisting}[language=problang, escapeinside={`}{'},
    basicstyle=\fontfamily{cmtt}\normalsize, numbers=left,
    numberstyle=\scriptsize\color{gray}, xleftmargin=0.4cm, numbers=none, caption={A symbolic formula for likelihood. }, label={form:likelihood}]
`$\formula_{\ell}(x) := \bigwedge\limits_{j=1}^{n} \mathcal{M}\big(x, \colorbox{red!50}{\textcolor{white}{$w_{j}$}}\big)$'
\end{lstlisting}
Because $\formula_{\ell}(x)$ is a pure conjunction, its satisfaction probability factors into the product of the probabilities of the observed token events. This gives the standard likelihood objective as a special case (see \citet{richardsoncots} for a similar result).
\begin{restatable}[\textcolor{blue}{Likelihood special case}]{prop}{likelihood}
%\begin{prop}[\textcolor{blue}{Likelihood special case}]
If each event probability is identified with the model's next-token probability,
$\mathbb{P}_{\theta}(\mathcal{M}(x,\colorbox{red!50}{\textcolor{white}{$w_{j}$}})) = \mathbb{P}_{\mathcal{M}_{\theta}}(\colorbox{red!50}{\textcolor{white}{$w_{j}$}} \mid x_{<j})$,
then the satisfaction probability of $\formula_{\ell}(x)$ is equal to the language-model likelihood of $x$:
$$
\mathbb{P}_{\theta}(\formula_{\ell}(x))
=
\prod_{j=1}^{n}
\mathbb{P}_{\mathcal{M}_{\theta}}(\colorbox{red!50}{\textcolor{white}{$w_{j}$}} \mid x_{<j}).
$$
Consequently, the semantic loss $\ell_{\mathrm{sl}}(\formula_{\ell}(x);\theta)$ is the negative log-likelihood, i.e., the (unnormalized) \textbf{\emph{cross-entropy loss}} $\ell_{\text{ce}}(\formula_{\ell}(x), \theta) = -\log P_{\mathcal{M}_{\theta}}(x)$.
\label{prop:likelihood}
\end{restatable}
%\end{prop}
\end{example}
Thus, ordinary next-token training is recovered in \modellog \update{in the case where the} declarative theory is a conjunction of observed token events. \update{We note that other likelihood-like losses, such as unlikelihood loss from \citet{welleck2019neural}, can be expressed as a similar conjunctive formula extended with logical negation.}  This restates the \textbf{central puzzle of cross-entropy} in semantic terms: \emph{why should optimizing satisfaction of this very particular kind of formula -- \update{a simple conjunction of prediction events} -- produce models that satisfy much richer semantic constraints?} Later, we argue that part of the answer \update{may lie} in the inherent informativeness of likelihood formulas. To make this idea precise, we will compare a model’s satisfaction probability against the probability of satisfying the same formula by chance. This motivates an \textbf{uninformed prior}: the weighted model count of a formula $\formula$ under uniform weights, where each prediction variable has probability \texttt{0.5}:
\begin{align}
\text{WMC}_{0}(\formula) := \frac{\mid \mathsf{I}(\formula) \mid}{2^{\mid \mathrm{vars}(\formula) \mid}}.
\label{eq:uniformed_baseline}
\end{align}
Here, $\mathrm{vars}(\formula)$ is the set of \update{atomic} variables appearing in $\formula$. As illustrated in Figure~\ref{fig:wmc}\textbf{(E)}, this quantity captures the baseline satisfiability of the formula independent of a model, and it becomes central to the issues we discuss next and our evaluation. 

\subsection{Subtleties in pre-training evaluation}
\label{sec:main_issues}

The semantics above is deliberately close to standard neuro-symbolic formulations \citep{manhaeve2018deepproblog, xu2018semantic} that have been applied to language model training \citep{ahmed2023pseudo, richardson2024understanding, calanzone2025logically}. In \update{these settings}, the main role of the probabilistic semantics is to provide a differentiable signal for learning; \update{probabilities, such as the raw token likelihoods considered above, therefore need only define a useful training objective, not an interpretable or calibrated evaluation score}. In \modellog, our initial goal is measurement: $\mathbb{P}_{\theta}(\formula)$ is meant to quantify how strongly a pre-trained model satisfies a semantic constraint. This evaluation setting makes interpretability of the probabilities themselves essential, leading to the following issues involving calibration and constraint selection. 

\paragraph{What do token variables mean semantically?} Treating local token decisions as logical variables makes them Bernoulli-like, i.e., either true/false in a given context. Following \citet{richardson2024understanding}, we interpret truth as local validity, or whether a token is an acceptable continuation and its probability exceeds some acceptability threshold $\epsilon$. The difficulty is that language models output categorical distributions over the vocabulary, not Bernoulli probabilities over validity events. Thus, raw token probabilities must be calibrated into event probabilities that reflect membership in the model’s locally valid region of continuations. 

We \textbf{calibrate token event probabilities} by using structural information about the local distribution to offset the odds of the raw token probability: 
\begin{align}
\hat{\mathbb{P}}_{\theta_{C}}\big({\scriptsize\mathcal{M}(x, \colorbox{red!50}{\textcolor{white}{$w_{j}$}})}\big) \hspace{-.1cm} &:= \sigma\bigg( \text{logit}\big(\mathbb{P}_{\mathcal{M}_{\theta}}({\scriptsize \colorbox{red!50}{\textcolor{white}{$w_{j}$}} \mid x})\big) + C \bigg) \nonumber \\ 
       &\propto e^{C} \mathbb{P}_{\mathcal{M}_{\theta}}({\scriptsize \colorbox{red!50}{\textcolor{white}{$w_{j}$}} \mid x}).
\label{eq:calibration}
\end{align}
Here, $\text{logit}(\cdot)$ corresponds to the log odds of the raw token probability and $C$ is an offset term representing the size of the valid region of continuations. Intuitively, larger values of $C$ increase the odds of a token being treated as valid, correcting for the fact that raw probability mass may be spread across many plausible continuations. The calibrated value therefore estimates live-region membership rather than raw next-token likelihood.

While different choices can be made for $C$, we use the entropy of the top-$p$ tokens or \textbf{nucleus} of the local distribution \citep{holtzman2019curious}, which we later refer to as \textbf{nucleus entropy calibration} and discuss below through an example. 

\begin{example}[\textcolor{blue}{Token calibration via nucleus entropy}]
For the prefixes \textcolor{gray}{\emph{A computer \underline{is not} / \underline{is} an \_\_}}, the completion \emph{airplane} is semantically valid only in the negative context. Yet under \texttt{SmolLM2-135M} \citep{lozhkovsmollm2}, its raw probabilities are small in both contexts: \texttt{0.0055} after \textcolor{gray}{\emph{\underline{is not}}} and \texttt{3.2678e-05} after \textcolor{gray}{\emph{\underline{is}}}. Nucleus entropy calibration rescales these probabilities by the effective size of the local live region: entropies of \texttt{3.97} and \texttt{4.11} for the top-$p$ tokens (with nucleus $p=0.95$) correspond to perplexities, $e^H$, of roughly \texttt{53} and \texttt{61}, estimating the number of plausible alternatives in the nucleus. This yields calibrated event probabilities of \texttt{0.58} and \texttt{0.001}, recovering the intended contrast: \emph{airplane} is plausible in the negative context but not the affirmative one.
\end{example}

\paragraph{What do constraint probabilities tell us?}
\update{A second issue is that raw constraint probabilities conflate model behavior with formula structure: $\propvar_{1} \, \logicor \, \propvar_{2}$ can receive higher raw probability than $\propvar_{1} \, \logicand \, \propvar_{2}$ because it admits more satisfying assignments}. In learning, \update{analogous effects appear as \emph{reasoning shortcuts}} \citep{van2024independence, marconato2023not, marconato2025symbol}, \update{which} create spurious learning patterns; such issues arise in evaluation too. We therefore compare model-weighted satisfaction to the uninformed baseline $\mathrm{WMC}_{0}(\formula)$  from Eq.~\ref{eq:uniformed_baseline} and define our main \textbf{probabilistic consistency metric} $\rho$: 
\begin{align}
\rho(\formula;\theta_{C}) = \frac{\mathrm{WMC}(\formula; \theta_{C}) - \mathrm{WMC}_{0}(\formula)}{1 - \mathrm{WMC}_{0}(\formula)}.
\label{eq:prob_consistency}
\end{align}
This measures the \update{fraction of} possible improvement over chance satisfaction \update{achieved by the model}: $\rho=0$ matches the uninformed baseline, $\rho=1$ is perfect satisfaction. We also use a hard, \textbf{accuracy-like consistency metric},
\begin{align}
\mathrm{CAcc}(\formula;\theta_{C})
= \mathbbm{1}\!\left[
\mathrm{WMC}(\formula;\theta_{C}) > \mathrm{WMC}_{0}(\formula)\right],
\label{eq:cacc}
\end{align}
which records whether the model satisfies the constraint better than chance. Averaging this over examples gives a dataset-level score, which we report later in \S~\ref{sec:findings} and Table~\ref{tab:wmc_main_agg}.

\paragraph{How can we determine if a constraint is useful?} In \modellog, experiment designers must come up with specific constraints to test. A natural question then is: \emph{how can we know if a constraint is meaningful to use for testing?} Intuitively, one should prioritize constraints that are inherently informative and hence not subject to spurious satisfaction. Using the uninformed baseline, we define the notion of \textbf{constraint informativeness}: $I(\formula) := -\log_{2} \mathrm{WMC}_{0}(\formula)$.  

Here $I(\formula)$ measures how many bits of information are gained by knowing that $\formula$ holds under an otherwise uniform assignment. As we discuss in the example below, this gives us a useful tool for reasoning about what to test and how it might connect to learning. 

\begin{example}[\textcolor{blue}{Constraint informativeness}]
For a sequence $x = w_{1}, \ldots, w_{n}$, the likelihood formula $\formula_{\ell}(x)$ from Formula~\ref{form:likelihood} requires every observed token-prediction event $\mathcal{M}(x, {\scriptsize \colorbox{red!50}{\textcolor{white}{$w_{j}$}}})$ to hold. It is therefore a complete assignment, or minterm, and is maximally informative among satisfiable constraints over these variables (see \S~\ref{app:proofs}).

\begin{restatable}[\textcolor{blue}{Likelihood is a maximally informative minterm}]{prop}{maxinformativeness}
Let $\formula_{\ell}(x)$ be the likelihood formula for a text sequence $x$ over $n$ observed token-prediction events. For any satisfiable formula $\formula(x)$ over the same variables: 
$$I(\formula(x)) \leq I(\formula_{\ell}(x)) = n.$$
Moreover, equality holds iff $\formula(x)$ has exactly one satisfying interpretation.
\label{prop:cemax}
\end{restatable}

Thus, likelihood is maximally informative \update{because it corresponds to a complete assignment of the observed token events}. Richer constraints, such as implications or biconditionals, may test different structure but typically admit more satisfying interpretations and are therefore less informative in this formal sense. This reframes likelihood-style next token prediction as optimizing satisfaction of a maximally informative positive constraint from the text.
\end{example}

More generally, informativeness provides a way to reason about the logical strength of evaluation constraints. If one formula \emph{semantically entails} ($\models$) another, then satisfying the stronger formula also guarantees satisfaction of the weaker one. This gives the following monotonicity property:
\begin{restatable}[\textcolor{blue}{Monotonicity of satisfaction and informativeness}]{prop}{monotonicity}
For any two formulas $\formula_{1}$ and $\formula_{2}$ over the same variables, if $\formula_{1} \models \formula_{2}$, then
$$\mathbb{P}_{\theta}(\formula_{1}) \leq \mathbb{P}_{\theta}(\formula_{2})
\quad\text{and}\quad I(\formula_{1}) \geq I(\formula_{2}).
$$
\end{restatable}
This highlights an important difference between learning and evaluation. For learning, highly informative constraints may be desirable because they can induce stronger losses, \update{a topic we revisit in Section~\ref{sec:discussion}}. For evaluation, however, the most informative constraint is not always the most useful diagnostic. A very strong formula may simply fail, while weaker entailed constraints can reveal partial semantic knowledge and provide a more interpretable picture of model behavior. \update{As a simple but useful consequence, monotonicity can be used constructively: by weakening the likelihood conjunction, we obtain likelihood-style upper bounds that provide principled new evaluation targets.}

\begin{example}[\textcolor{blue}{Deriving novel likelihood bounds}]
\update{Let $\bigvee_{\geq k}$ denote a threshold disjunction, true when at least $k$ of its arguments are true:}
\begin{lstlisting}[language=problang, escapeinside={`}{'},
    basicstyle=\fontfamily{cmtt}\normalsize, numbers=none,
    caption={A relaxed form of Formula~\ref{form:likelihood}.}, label={form:at_least_k}]
`$\formula_{\geq k}(x) := \sideset{}{_{\geq k}}\bigvee\limits_{j=1}^{n} \mathcal{M}(x,\colorbox{red!50}{\textcolor{white}{$w_j$}})$'
\end{lstlisting}
Here $k=1$ gives ordinary disjunction ($\logicor$) and $k=n$ gives $\formula_{\ell}(x)$, which is ordinary sequence likelihood under the raw weighting from Prop.~\ref{prop:likelihood}. Since $\formula_{\geq n}(x)\models \cdots \models\formula_{\geq 1}(x)$, monotonicity gives a hierarchy of increasingly tight upper bounds:
$$
\underbrace{\mathbb{P}_{\theta}(\formula_{\geq 1}(x))}_{\logicor}
\geq
\mathbb{P}_{\theta}(\formula_{\geq 2}(x))
\geq
... 
\geq
\underbrace{\mathbb{P}_{\theta}(\formula_{\geq n}(x))}_{\logicand}.
$$
\update{While simple, this shows how bounds that are not obvious from the usual product form of likelihood become immediate once likelihood is represented semantically as a logical conjunction. This relates in spirit to selective-modeling objectives that train on subsets of tokens \citep{lin2024rho}, and shows how such relaxations can be derived semantically.}
\end{example}

\begin{table*}
\centering 
{\scriptsize
\begin{tabular}{| l l l l |}
        \hline 
        \multicolumn{1}{|c}{\textbf{Task Subset}} & \multicolumn{1}{c}{\textbf{Constraint } $\formula$}  & $I(\formula) \approx$ & \multicolumn{1}{c|}{\textbf{Example}} \\ \hline 
        \textbf{Transitivity} & \begin{tabular}[t]{@{}l@{}}$(\mathcal{M}(x_1, \colorbox{red!50}{\tiny\bfseries\color{white}$y_1$}) \land \mathcal{M}(x_2, \colorbox{blue!50}{\tiny\bfseries\color{white}$y_2$}))$ \\ $\implication \mathcal{M}(x_1, \colorbox{blue!50}{\tiny\bfseries\color{white}$y_2$})$\end{tabular} & \texttt{0.193} & \begin{tabular}[t]{@{}l@{}}A Golden Retriever is a $\colorbox{red!50}{\tiny\bfseries\color{white}dog}$ $\land$ Every dog is a $\colorbox{blue!50}{\tiny\bfseries\color{white}mammal}$ \\ $\implication$ A Golden Retriever is a $\colorbox{blue!50}{\tiny\bfseries\color{white}mammal}$\end{tabular} \\ 
        \textbf{Forward Implication} & $\mathcal{M}(x, \colorbox{red!50}{\tiny\bfseries\color{white}$y_1$}) \implication \mathcal{M}(x, \colorbox{blue!50}{\tiny\bfseries\color{white}$y_2$})$ & \texttt{0.415} & A Golden Retriever is a $\colorbox{red!50}{\tiny\bfseries\color{white}dog}$ $\implication$ A Golden Retriever is a $\colorbox{blue!50}{\tiny\bfseries\color{white}mammal}$ \\ 
        \textbf{Negation Consistency} & $\mathcal{M}(x, \colorbox{red!50}{\tiny\bfseries\color{white}$y$}) \oplus \mathcal{M}(\neg x, \colorbox{red!50}{\tiny\bfseries\color{white}$y$})$ & \texttt{1.0} & A Golden Retriever is a $\colorbox{red!50}{\tiny\bfseries\color{white}dog}$ $\oplus$ A Golden Retriever is not a $\colorbox{red!50}{\tiny\bfseries\color{white}dog}$ \\ \hline 
        Entity \textbf{Mutual Exclusivity} & \multirow[c]{2}{*}{$\mathcal{M}(x, \colorbox{red!50}{\tiny\bfseries\color{white}$y_1$}) \oplus \mathcal{M}(x, \colorbox{blue!50}{\tiny\bfseries\color{white}$y_2$})$} & \multirow[c]{2}{*}{\texttt{1.0}} & A Golden Retriever is a $\colorbox{red!50}{\tiny\bfseries\color{white}mammal}$ $\oplus$ A Golden Retriever is a $\colorbox{blue!50}{\tiny\bfseries\color{white}reptile}$ \\ 
        \textbf{Spatial} Mutual Exclusivity &  & & Paris is in $\colorbox{red!50}{\tiny\bfseries\color{white}Germany}$ $\oplus$ Paris is in $\colorbox{blue!50}{\tiny\bfseries\color{white}France}$ \\ \hline
    \end{tabular}
}
\caption{A description of the \probCT diagnostic benchmark in terms of its five \textbf{task subsets}, the structure of the \textbf{constraints} $\formula$ being tested with their informativeness $I(\formula)$ and an \textbf{example} set of prefixes. 
}
\label{tab:pct_examples}
\end{table*}

\begin{table*}[t]
\centering
{\footnotesize
\setlength{\tabcolsep}{6pt}

  \begin{tabular}{lccccc}
  \toprule
  Model & Transitivity & Forward Implication & Neg.\ Consistency & Mutual Exclusivity & Spatial
  Exclusivity \\
  \midrule
  \texttt{Gemma-3-270M} & \cellcolor{green!10}{$\mathbf{81.0} (\pm \textcolor{gray}{2.3})$} &
  \cellcolor{green!8}{$79.7 (\pm \textcolor{gray}{0.5})$} & $18.5 (\pm \textcolor{gray}{1.4})$ &
  $48.7 (\pm \textcolor{gray}{0.8})$ & $18.5 (\pm \textcolor{gray}{4.8})$ \\
  \texttt{Gemma-3-1B-pt} & $69.0 (\pm \textcolor{gray}{1.5})$ & $65.1 (\pm \textcolor{gray}{1.5})$ & $15.9
  (\pm \textcolor{gray}{0.8})$ & $57.3 (\pm \textcolor{gray}{3.9})$ & $12.4 (\pm
  \textcolor{gray}{4.8})$ \\
  \texttt{Gemma-3-4B-pt} & $46.9 (\pm \textcolor{gray}{1.1})$ & $50.8 (\pm \textcolor{gray}{1.7})$ & $18.2
  (\pm \textcolor{gray}{0.9})$ & $56.8 (\pm \textcolor{gray}{5.2})$ & $19.9 (\pm
  \textcolor{gray}{6.6})$ \\
  \texttt{Gemma-3-12B-pt} & \cellcolor{ai2lightpink} $40.4 (\pm \textcolor{gray}{1.6})$ & \cellcolor{ai2lightpink} $45.2 (\pm \textcolor{gray}{2.0})$ & $19.2
  (\pm \textcolor{gray}{0.9})$ & $54.7 (\pm \textcolor{gray}{5.4})$ & $23.8 (\pm
  \textcolor{gray}{7.7})$ \\
  \midrule
  \texttt{Llama-3.2-1B} & \cellcolor{green!8}{$52.7 (\pm \textcolor{gray}{2.4})$} &
  \cellcolor{green!8}{$68.7 (\pm \textcolor{gray}{1.8})$} & \cellcolor{green!8}{$\mathbf{23.1}
  (\pm \textcolor{gray}{2.1})$} & \cellcolor{green!8}{$64.0 (\pm \textcolor{gray}{3.8})$} &
  \cellcolor{green!8}{$\mathbf{31.5} (\pm \textcolor{gray}{7.4})$} \\
  \texttt{Llama-3.2-3B} & $45.5 (\pm \textcolor{gray}{1.3})$ & $60.8 (\pm \textcolor{gray}{2.1})$ & \cellcolor{ai2lightpink} $21.6
  (\pm \textcolor{gray}{2.6})$ & $61.6 (\pm \textcolor{gray}{6.0})$ & $25.3 (\pm
  \textcolor{gray}{6.9})$ \\
  \texttt{Llama-3.1-8B} & \cellcolor{ai2lightpink} $35.2 (\pm \textcolor{gray}{1.5})$ & \cellcolor{ai2lightpink} $50.0 (\pm \textcolor{gray}{1.2})$ & $22.7
  (\pm \textcolor{gray}{1.3})$ & \cellcolor{ai2lightpink} $60.0 (\pm \textcolor{gray}{5.9})$ & \cellcolor{ai2lightpink} $16.1 (\pm
  \textcolor{gray}{6.4})$ \\
  \midrule
  \texttt{Qwen3-0.6B} & \cellcolor{green!8}{$76.4 (\pm \textcolor{gray}{0.5})$} &
  \cellcolor{green!8}{$\mathbf{82.8} (\pm \textcolor{gray}{1.5})$} & $17.2 (\pm
  \textcolor{gray}{0.6})$ & $58.2 (\pm \textcolor{gray}{0.7})$ & $9.1 (\pm \textcolor{gray}{3.2})$
  \\
  \texttt{Qwen3-1.7B} & $64.2 (\pm \textcolor{gray}{1.2})$ & $75.3 (\pm \textcolor{gray}{1.6})$ & $18.9 (\pm
  \textcolor{gray}{3.1})$ & $62.9 (\pm \textcolor{gray}{2.1})$ & $25.9 (\pm \textcolor{gray}{7.2})$
  \\
  \texttt{Qwen3-4B} & $60.8 (\pm \textcolor{gray}{3.2})$ & $73.0 (\pm \textcolor{gray}{2.6})$ & $19.3 (\pm
  \textcolor{gray}{0.6})$ & $67.7 (\pm \textcolor{gray}{5.0})$ & $16.6 (\pm \textcolor{gray}{5.1})$
  \\
  \texttt{Qwen3-8B} & \cellcolor{ai2lightpink} $53.7 (\pm \textcolor{gray}{2.2})$ & $69.6 (\pm \textcolor{gray}{2.7})$ & $20.8 (\pm
  \textcolor{gray}{0.1})$ & $\mathbf{68.2} (\pm \textcolor{gray}{4.3})$ & $14.8 (\pm
  \textcolor{gray}{5.3})$ \\
  \texttt{Qwen3-14B} & $56.4 (\pm \textcolor{gray}{2.1})$ & \cellcolor{ai2lightpink} $64.1 (\pm \textcolor{gray}{3.6})$ & $20.9 (\pm
  \textcolor{gray}{1.2})$ & $65.9 (\pm \textcolor{gray}{5.6})$ & $12.6 (\pm \textcolor{gray}{5.4})$
  \\
  \bottomrule
  \end{tabular}
}
\caption{Main results (\text{CAcc} \%) across the different reasoning tasks in \probCT and model families, averaged over different nucleus values $p \in \{0.8, 0.85, 0.9, 0.95\}$ used for calibration.}
\label{tab:wmc_main_agg}
\end{table*}

\section{Testing the framework}
\label{sec:testing}

To evaluate \modellog empirically, we introduce \probCT (\textbf{Prob}abilistic \textbf{C}onsistency \textbf{T}ests), a diagnostic suite described in \S~\ref{sec:probct}. We then describe our experimental setup in \S~\ref{sec:exp}\footnote{\update{Our code and data are available at
\url{https://github.com/cullena20/ModelLog} % PUBLIC
%\url{https://anonymous.4open.science/r/ModelLog-458E} % ANONYMOUS
}.}.

\subsection{A case study on \probCT}
\label{sec:probct}

Following prior work on knowledge probing \update{\citep{petroni2019language, kassner2020negated, elazar2021measuring}}, we automatically generate tasks from factual knowledge graphs (\S~\ref{sec:kbs}). As shown in Table~\ref{tab:pct_examples}, \probCT contains five task subsets, inspired by prior work, that cover constraints with different logical structures and levels of informativeness: \textbf{transitivity} \citep{li2019logic}, \textbf{forward implication} \citep{richardson2020probing}, \textbf{negation consistency} \citep{kassner2020negated}, and two forms of \textbf{mutual exclusivity} involving exclusive-or ($\oplus$) reasoning \citep{xu2018semantic}. 

We apply \update{LLM-as-a-judge style} filtering with \texttt{GPT-5o-mini} to improve correctness and naturalness of the generated data (\S~\ref{sec:llm_judge}), \update{yielding around \texttt{500-1000} examples per subset.} As a check on quality, a subset of the authors validated 50 randomly sampled examples, finding overall \update{high quality and acceptance} (\S~\ref{sec:human_ann}). Finally, we restrict examples to single-token substitutions, producing the model-specific \update{splits further detailed in \S~\ref{sec:sizes}}.

\subsection{Models and Experimental Setup}
\label{sec:exp}

We evaluate the following pre-trained models: \texttt{Gemma-3} \citep{gemmateam2025gemma3technicalreport}, \texttt{Llama-3} \citep{grattafiori2024llama}, and \texttt{Qwen3} \citep{yang2025qwen3}, \update{with parameter sizes from} \texttt{270M} to \texttt{14B}. For each model and \probCT subset, we report average constraint accuracy \textrm{CAcc} (Eq.~\ref{eq:cacc}) and per-example probabilistic consistency $\rho$ (Eq.~\ref{eq:prob_consistency}). Unless noted, all results use nucleus-entropy calibration (\S~\ref{sec:main_issues}); for \textrm{CAcc}, we report mean performance across different nucleus thresholds $p$ with standard deviations.

\section{Results and Findings}
\label{sec:findings}

Our main results on \probCT are reported in Table~\ref{tab:wmc_main_agg} and Figure~\ref{fig:qwen_base_constraints_main} (see also \S~\ref{sec:histograms}). 

\begin{figure*}[t]
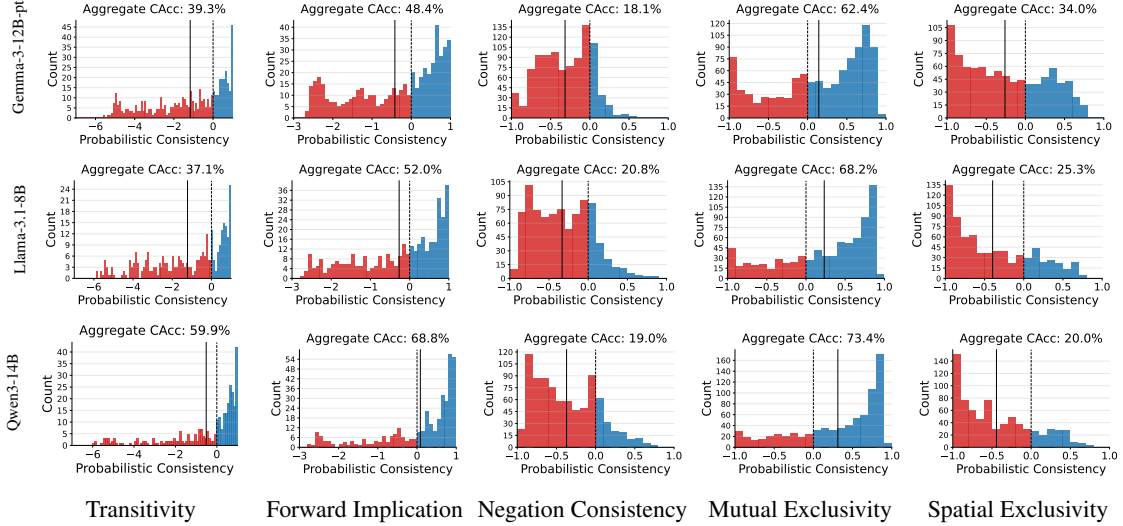

\centering

\wmcplotrow{Gemma-3-12B-pt}{cal_nucleus_entropy095}
\wmcplotrow{Llama-3.1-8B}{cal_nucleus_entropy095}
\wmcplotrowlabeled{Qwen3-14B}{cal_nucleus_entropy095}

\caption{Probabilistic Consistency score histograms $\rho(\cdot, \theta_{C})$ for our largest models in each model family, using nucleus value $p=0.95$. Points highlighted in red are inconsistent, while points highlighted in blue are consistent. The solid vertical black line corresponds to the average probabilistic consistency (see also \S~\ref{sec:histograms}).}
\label{fig:qwen_base_constraints_main}
\end{figure*}

\paragraph{Models often fail declarative consistency tests.}

Across model families, performance is low on several \probCT subsets, especially on \textbf{negation consistency} and \textbf{spatial exclusivity}. For example, in the latter case, the largest \texttt{Gemma-3} and \texttt{Qwen3} models improve over the uninformed baseline on only \texttt{23.8}\% and \texttt{12.6}\% of examples, respectively. Figure~\ref{fig:qwen_base_constraints_main} shows that these failures are not only aggregate effects: many examples have negative probabilistic consistency scores $\rho$, meaning that the model supports the constraint less than the uninformed baseline. Thus, \modellog reveals failures not simply in individual token predictions, but in the semantic relations among predictions.

\begin{table}[t]
\centering
\small
\setlength{\tabcolsep}{2pt}
{\footnotesize
\begin{tabular}{lccc}
\toprule
Model & $\mathbb{P}_{\theta_{C}}(\propvar_{2}){-}\mathbb{P}_{\theta_{C}}(\propvar_{1})$ & $\mathbb{P}_{\theta_{C}}(\propvar_{2}{>}\propvar_{1})$ & $\textrm{CAcc}$ \\
\midrule
\texttt{Gemma-3-270M}   & $0.17$  & $64.7$ & $75.4$ \\
\texttt{Gemma-3-1B-pt}  & $-0.02$ & $49.5$ & $57.8$ \\
\texttt{Gemma-3-4B-pt}  & $-0.19$ & $34.7$ & $44.1$ \\
\texttt{Gemma-3-12B-pt} & $-0.24$ & $30.1$ & $38.3$ \\
\bottomrule
\end{tabular}
}
\caption{Inverse-scaling \citep{mckenzie2023inverse} \update{behavior} for the forward implication: $\propvar_{1} \implication \propvar_{2}$ for \texttt{Gemma-3}. $\mathbb{P}_{\theta_{C}}(\propvar_{2}){-}\mathbb{P}_{\theta_{C}}(\propvar_{1})$ is the mean difference between conclusion and premise probabilities. $\mathbb{P}_{\theta_{C}}(\propvar_{2}{>}\propvar_{1})$ is the percentage of points where $\mathbb{P}_{\theta_{C}}(\propvar_{2})>\mathbb{P}_{\theta_{C}}(\propvar_{1})$. \update{See \S~\ref{sec:inv_scaling} for similar patterns across our full set of models.}
}
\label{tab:fi_diagnostic_gemma}
\end{table}
\paragraph{Consistency can inverse-scale.}

A more surprising pattern is that larger models are not always more consistent. On \textbf{transitivity} and \textbf{forward implication}, \textrm{CAcc} often decreases with model scale across all three families. Further analysis suggests that this reflects systematic changes in how models distribute probability across premise and conclusion facts (see an example in Table~\ref{tab:inv_scale_examples}). For implication constraints $\propvar_{1} \implication \propvar_{2}$, where $\propvar_{1}$ is a more specific fact and $\propvar_{2}$ a broader consequence, smaller models are more likely to assign higher probability to $\propvar_{2}$ than to $\propvar_{1}$. As shown in Table~\ref{tab:fi_diagnostic_gemma}, this trend reverses with scale: for example, $\mathbb{P}_{\theta_{C}}(\propvar_{2})-\mathbb{P}_{\theta_{C}}(\propvar_{1})$ drops from \texttt{0.17} to \texttt{-0.24} across \texttt{Gemma-3}. Similar patterns are found for our other models (Table~\ref{tab:fi_diagnostic}) and for the \textbf{transitivity} rule (Table~\ref{tab:tr_diagnostic}).

\update{Importantly, this should not be read as showing that larger models are worse reasoners overall. Larger models may distribute probability differently across logically related predictions, sometimes reducing consistency under one constraint while improving other aspects of robustness. \modellog makes such shifts measurable, providing a framework for testing how model scale changes the semantic structure of probabilistic predictions.}

\paragraph{Findings are stable across calibration choices.}

\update{Results are stable across nucleus thresholds, suggesting that the observed patterns are not artifacts of calibration \update{(for results across different calibration methods, see \S~\ref{sec:other_calibration})}. Taken together, these case studies show that \modellog can reveal systematic failures, inverse-scaling patterns, and qualitative changes in model behavior that are difficult to capture with standard evaluation metrics.}

\section{Discussion and Future Directions}
\label{sec:discussion}

\update{We began with the puzzle of cross-entropy: \emph{why does a local next-token objective provide such an effective learning signal, and can richer semantic losses improve pre-training?} Although we focused on evaluation and the semantics of model constraints, \modellog defines evaluation scores as differentiable satisfaction probabilities, connected through the semantic loss $\ell_{\text{sl}}(\formula; \theta)$ (Eq.~\ref{eq:sl}). We end by returning to this learning question, showing that concepts introduced for evaluation, such as the uninformed baseline $\mathrm{WMC}_{0}(\formula)$ and constraint informativeness $I(\formula)$, reappear directly in the gradients of semantic loss $\nabla \ell_{\text{sl}}(\formula; \theta)$. This suggests that they are not only diagnostic quantities, but also potential tools for studying how the evaluation failures observed above might be turned into more effective training strategies.}

\paragraph{Informativeness shapes semantic-loss gradients.}
The first connection is that constraint informativeness contributes to the potential scale of the semantic-loss gradient. Since gradients are vector-valued, we measure this scale standardly using a norm. Highly informative constraints are hard to satisfy by chance, so failures on such constraints can receive larger loss amplification. More precisely, we have the following bound (see proof in \S~\ref{app:proofs}):
\begin{restatable}[\textcolor{blue}{Gradient scaling by informativeness}]{prop}{gradientscale}
%\begin{prop}[\textcolor{blue}{Gradient scaling by informativeness}]
For any formula $\formula$ such that $\mathrm{WMC}(\formula;\theta) \geq \mathrm{WMC}_{0}(\formula)$,
$$
\left\Vert \nabla_{\theta}\ell_{\mathrm{sl}}(\formula;\theta) \right\Vert
\leq
2^{I(\formula)}
\left\Vert
\nabla_{\theta}\mathrm{WMC}(\formula;\theta)
\right\Vert.
$$
\end{restatable}
%\end{prop}

This bound shows that informativeness appears as an exponential amplification term in the possible learning signal induced by a constraint. It does not imply that more informative constraints always produce larger or better updates: the actual gradient can be reweighted externally and also depends on the local derivative of the satisfaction probability, which may point in different directions across noisy examples. Still, the result suggests that highly informative constraints, such as the likelihood constraint (Formula~\ref{form:likelihood}), can provide strong learning signals when their gradients are coherent. This may therefore help explain part of their success.

\paragraph{Informativeness appears directly in the gradient.}
The previous bound controls gradient scale; the next identity decomposes the gradient into interpretable factors. First, we define the following:
\[
A(\formula;\theta)
=
\frac{\mathrm{WMC}_{0}(\formula)}
{\mathrm{WMC}(\formula;\theta)}.
\]
This compares the model's satisfaction probability to the uninformed baseline: $A(\formula;\theta)<1$ means that the model satisfies the constraint better than chance, while $A(\formula;\theta)>1$ means that it does worse. Based on this, the semantic-loss gradient then factors as follows:
%\begin{prop}[\textcolor{blue}{Semantic-loss gradient identity}]
\begin{restatable}[\textcolor{blue}{Semantic-loss gradient identity}]{prop}{slidentity}
For any formula $\formula$,
$$
\nabla_{\theta}\ell_{\mathrm{sl}}(\formula;\theta) =-A(\formula;\theta)\, 2^{I(\formula)} \,
\nabla_{\theta}\mathrm{WMC}(\formula;\theta).
$$
\end{restatable}
%\end{prop}
\noindent Thus, the gradient decomposes into three pieces: a fit term $A(\formula;\theta)$, an intrinsic informativeness term $2^{I(\formula)}$, and a local direction term $\nabla_{\theta}\mathrm{WMC}(\formula;\theta)$. This gives a more nuanced view of semantic learning signals: informativeness appears as a structural factor, while $A(\formula;\theta)$ reflects how much the model currently under- or over-satisfies the constraint relative to the uninformed baseline. The actual update direction then remains determined by the local derivative $\nabla_{\theta}\mathrm{WMC}(\formula;\theta)$. 

The informativeness term again helps explain why likelihood-style formulas can induce strong learning signals. More importantly, by separating constraint strength from model fit and local update direction, the identity makes the latter terms diagnostic of whether a constraint provides useful learning evidence.

\paragraph{Putting the pieces together.}

\update{These results give a partial answer to the cross-entropy puzzle. Likelihood is recovered in \modellog as a conjunction of observed token-prediction events whose satisfaction probability is sequence likelihood (Prop.~\ref{prop:likelihood}). This formula is maximally informative (Prop.~\ref{prop:cemax}), and the gradient results show that informativeness can amplify the semantic-loss signal. Thus, likelihood-style objectives may be effective because they optimize highly informative constraints supplied by text, even though they do not directly encode richer semantic relations.}

\update{At the same time, these results complicate a simple training story in which semantic failures are fixed by replacing likelihood with weaker constraint-based losses: since likelihood-style constraints are already maximally informative, such replacements may not provide stronger pre-training signals. Richer constraints may instead be most useful for identifying which examples or constraint types provide useful learning evidence. Future work can use the gradient decomposition above to study when such constraints provide coherent learning signals, connecting evaluation failures to questions of constraint informativeness, data coverage, selection, and weighting.}

\section{Conclusion}

\update{We introduced \modellog, a declarative probabilistic framework for evaluating pre-trained language models through logical constraints over token-level predictions. Using exact probabilistic inference, \modellog measures graded satisfaction of semantic properties that are difficult to capture with likelihood or accuracy alone. Through \probCT, we showed how this framework can reveal consistency failures in current models, and our formal analysis characterized the structure and informativeness of evaluation constraints, including their connection to semantic-loss gradients. This suggests a broader role for evaluation as a methodology for making model behavior semantically measurable while also informing more interpretable learning objectives.}

\section*{Limitations}

This work has three main limitations. First, \modellog currently focuses on single-token prediction events and evaluates constraints under a factorized product model over local prediction events. This is a deliberate semantic abstraction: the events correspond to separate local queries to the language model, while dependencies among them are introduced declaratively by the constraints rather than assumed to be part of the model’s autoregressive distribution. This keeps inference tractable and enables clean formal analysis, but the current framework does not yet address multi-token spans, variable-length paraphrases, or full sequence-level constraints. Second, our empirical results are intended as small case studies of the framework and are based on synthetic diagnostic tasks generated from knowledge-graph relations, which, despite our efforts at filtering, may still contain noise. These tasks provide controlled tests of declarative constraints, but they do not capture the full complexity of natural language evaluation. Third, while our formal analysis connects probabilistic evaluation to semantic-loss gradients, we do not train models with these losses at scale. The learning discussion should therefore be read as a formal and diagnostic contribution, with actual training left for future work.

\section*{Acknowledgements} \update{We thank Kareem Ahmed, Gregor Betz, Yanai Elazar, Poorva Garg, Ronan Le Bras, William Merrill, Jackson Petty and Sahil Sidheekh for useful feedback at different stages of this work. We also thank Andrew McCallum and the Manning College of Information \& Computer Sciences at the University of Massachusetts Amherst for their support of this project, as well as the UMass Unity Cluster (\url{www.umass.edu/research-computing/unity-research-computing-platform}) for providing GPU compute. We acknowledge that both \texttt{ChatGPT} and \texttt{Claude} were used to improve some of the writing and overall presentation, and to provide feedback on some of the technical results. All mistakes remain our own.}

%\bibliography{custom}
% % Custom bibliography entries only
\bibliography{custom}

%\newpage
\appendix

\section{Proofs} 
\label{app:proofs}

\maxinformativeness*
\begin{proof}
Let $\formula$ be any satisfiable formula over $n$ variables. Since $\formula$, in virtue of being satisfiable, has at least one satisfying interpretation, it follows that: 
\[
|\mathsf{I}(\formula)| \geq 1.
\]
Under uniform weights,
\[
\mathrm{WMC}_{0}(\formula)
=
\frac{|\mathsf{I}(\formula)|}{2^n}
\geq
2^{-n}.
\]
Applying $-\log_2$ gives
\[
I(\formula)
=
-\log_2 \mathrm{WMC}_{0}(\formula)
\leq n.
\]
The likelihood formula $\formula_{\ell}$ has exactly one satisfying interpretation, namely the all-true assignment, so $\mathrm{WMC}_{0}(\formula_{\ell})=2^{-n}$ and $I(\formula_{\ell})=n$. Equality holds exactly when $|\mathsf{I}(\formula)|=1$.
\end{proof}

\monotonicity*
\begin{proof}
If $\formula_{1} \models \formula_{2}$, then every interpretation satisfying $\formula_{1}$ also satisfies $\formula_{2}$. By definition then  
$
\mathsf{I}(\formula_{1}) \subseteq \mathsf{I}(\formula_{2}).
$
Since $\mathbb{P}_{\theta}(\formula)$ is computed by summing nonnegative interpretation weights over $\mathsf{I}(\formula)$, it follows that:
$$
\mathbb{P}_{\theta}(\formula_{1}) \leq \mathbb{P}_{\theta}(\formula_{2}).
$$
The same subset relation holds under the uniform weighting used by $\mathrm{WMC}_{0}$, giving
$$
\mathrm{WMC}_{0}(\formula_{1}) \leq \mathrm{WMC}_{0}(\formula_{2}).
$$
Applying $-\log_{2}$ reverses the inequality, so that: 
\begin{align*}
I(\formula_{1}) &= -\log_{2}\mathrm{WMC}_{0}(\formula_{1}) \\ 
& \geq -\log_{2}\mathrm{WMC}_{0}(\formula_{2}) = I(\formula_{2}).
\end{align*}
\end{proof}

%% gradient scale
\gradientscale*
\begin{proof}
By applying the chain rule to the semantic loss, we have
\begin{align*}
\nabla_{\theta}\ell_{\mathrm{sl}}(\formula;\theta)
=
-\frac{1}{\mathrm{WMC}(\formula;\theta)}
\nabla_{\theta}\mathrm{WMC}(\formula;\theta).
\end{align*}
Taking norms gives
\begin{align*}
\left\Vert \nabla_{\theta}\ell_{\mathrm{sl}}(\formula;\theta) \right\Vert
=
\frac{1}{\mathrm{WMC}(\formula;\theta)}
\left\Vert \nabla_{\theta}\mathrm{WMC}(\formula;\theta) \right\Vert.
\end{align*}
By assuming then that $\mathrm{WMC}(\formula;\theta) \geq \mathrm{WMC}_{0}(\formula)$, it follows that
\[
\frac{1}{\mathrm{WMC}(\formula;\theta)}
\leq
\frac{1}{\mathrm{WMC}_{0}(\formula)}.
\]
Since $I(\formula)=-\log_2 \mathrm{WMC}_{0}(\formula)$, we have the following equivalence: 
\begin{align*}
\frac{1}{\mathrm{WMC}_{0}(\formula)}
=
2^{I(\formula)}.
\end{align*}
Combining these (in)equalities then yields: 
\[
\left\Vert \nabla_{\theta}\ell_{\mathrm{sl}}(\formula;\theta) \right\Vert
\leq
2^{I(\formula)}
\left\Vert \nabla_{\theta}\mathrm{WMC}(\formula;\theta) \right\Vert .
\]
\end{proof}

\slidentity*
\begin{proof}
This again relies on the chain rule from above for the semantic loss:
$$
\nabla_{\theta}\ell_{\mathrm{sl}}(\formula;\theta)
= -\frac{1}{\mathrm{WMC}(\formula;\theta)} \nabla_{\theta}\mathrm{WMC}(\formula;\theta).
$$
By definition,
\[
A(\formula;\theta)
=
\frac{\mathrm{WMC}_{0}(\formula)}
{\mathrm{WMC}(\formula;\theta)}.
\]
Importantly, the following holds:
\[
\frac{1}{\mathrm{WMC}(\formula;\theta)}
=
\frac{A(\formula;\theta)}
{\mathrm{WMC}_{0}(\formula)}.
\]
Since again we have: 
\[
I(\formula)=-\log_{2}\mathrm{WMC}_{0}(\formula),
\]
and
\[
\frac{1}{\mathrm{WMC}_{0}(\formula)}
=
2^{I(\formula)}.
\]
Substituting these identities into the semantic-loss gradient gives the desired output:
\[
\nabla_{\theta}\ell_{\mathrm{sl}}(\formula;\theta)
=
-
A(\formula;\theta)\,2^{I(\formula)}
\nabla_{\theta}\mathrm{WMC}(\formula;\theta).
\]
\end{proof}

\section{Dataset Generation}
\label{sec:dataset_generation_details}

\subsection{Large-Scale Dataset Generation via Knowledge Graphs}
\label{sec:kbs}

We construct evaluation data for each logical constraint from two
complementary knowledge bases (KBs): Open English WordNet\footnote{%
\url{https://github.com/globalwordnet/english-wordnet}} (OEWN)
\citep{mccrae-etal-2019-english}, which provides a large-scale
lexical hierarchy, and YAGO \citep{suchanek2007yago}, which
contributes physical and historical facts about named entities.

For each constraint type, we extract relational tuples from these KBs
and instantiate them into predefined natural language templates; for
example, a tuple relating an entity $e_1$ to its category $e_2$ is
rendered as ``A $e_1$ is a $e_2$.'' The KB's relational structure
then determines how individual facts are combined into constraint
instances, following the schemas in Table~\ref{tab:pct_examples}.
\textbf{Forward Implication} pairs two facts sharing an entity such
that the KB hierarchy entails one from the other. \textbf{Transitivity}
chains two such hierarchy edges to yield a three-fact example.
\textbf{Mutual Exclusivity} and \textbf{Spatial Mutual Exclusivity}
pair sibling nodes in the KB taxonomy that cannot jointly apply to a
single entity. \textbf{Negation Consistency} pairs a fact with its
negated form under the same templates. This procedure yields a large
candidate pool for each constraint, which is subsequently filtered through an LLM-As-Judge pipeline.

\subsection{LLM-As-Judge Filtering}
\label{sec:llm_judge}

Automatically constructing logical constraint datasets from a knowledge base may introduce numerous artifacts. To address this, we employ an LLM-As-Judge to filter the generated evaluation sets according to a strict criteria. Each extracted datapoint, which consists of a series of automatically extracted facts placed into predefined templates and a constraint between them, is scored according to the following binary criteria.

\begin{table*}[t]
\centering
\begin{tabular}{lrrrrr}
\toprule
Constraint & Raw & Filtered & \texttt{Gemma-3} & \texttt{Llama-3} & \texttt{Qwen3} \\
\midrule
Forward Implication   & 2{,}000 & 1{,}049 & 514 & 369 & 369 \\
Transitivity          & 2{,}000 & 1{,}091 & 486 & 307 & 307 \\
Mutual Exclusivity    & 2{,}000 & 1{,}468 & 975 & 781 & 781 \\
Negation Consistency  & 2{,}000 &    956  & 956 & 856 & 856 \\
Spatial Exclusivity   & 2{,}000 &    915  & 897 & 739 & 739 \\
\bottomrule
\end{tabular}
\caption{Dataset sizes at each filtering stage. \textit{Raw}: initial generated examples. \textit{Filtered}: after LLM-as-judge filtering. \texttt{Gemma-3}, \texttt{Llama-3}, \texttt{Qwen3}: remaining examples after removing multi-token prediction events per model family.}
\label{tab:data_size}
\end{table*}

\vspace{1em}

\begin{table*}[t]
\centering
\begin{tabular}{lcccc}
\toprule
Dataset & Human Accept & LLM Accept & Ann1 Precision & Ann2 Precision \\
\midrule
Forward Implication  & 0.78 & 0.44 & 0.91 & 0.95 \\
Mutual Exclusivity   & 0.84 & 0.76 & 0.95 & 0.97 \\
Negation             & 0.62 & 0.46 & 0.96 & 0.96 \\
Spatial Mutex        & 0.46 & 0.50 & 0.88 & 0.92 \\
Transitivity         & 0.86 & 0.56 & 1.00 & 1.00 \\
\bottomrule
\end{tabular}
\caption{Human validation of the KB to LLM filtering pipeline
on $n=50$ sampled examples per dataset. \textit{Human Accept} is the
fraction accepted by both annotators.
\textit{LLM Accept} is the fraction accepted by \texttt{GPT-5o-mini}.
\textit{Ann1 Precision} and \textit{Ann2 Precision} give the fraction
of LLM-accepted items each annotator independently judged correct
 i.e., the LLM filter's precision as measured by each annotator.}
\label{tab:human_annotation}
\end{table*}

\begin{enumerate}
    \item \textbf{Grammatical Coherence:} All facts must be grammatically natural. Because facts are extracted from knowledge graphs and inserted into predefined templates, grammatical errors and unnatural phrasings can emerge.
    \item \textbf{Monosemous:} All facts must concern entities with a single, unambiguous meaning. Knowledge graphs contain polysemous nodes whose extracted facts may be interpretable under multiple senses (e.g., ``bat'' may refer to the animal or the sports equipment), undermining the logical validity of the constraint.
    \item \textbf{Entity Salience:} All facts must concern well-known, real-world entities. Knowledge graphs sometimes contain overly specific, archaic, or obscure entries that, while technically valid graph nodes, produce evaluation examples that are unreasonable to expect a language model to have encountered during pre-training.
    \item \textbf{Factual Accuracy: } All facts must indeed be true. This controls for any errors that may exist in the knowledge graph.
\end{enumerate}

We use \texttt{GPT-5o-mini} as the LLM judge, and only keep datapoints passing each criteria.

\subsection{Dataset Sizes}
\label{sec:sizes}

To limit our evaluation to single-token substitutions, we additionally apply a final model-specific filtering step to remove any data points whose trigger word is split into multiple tokens.  We provide the size of each dataset in \probCT at each filtering stage in Table \ref{tab:data_size}. This consists of the raw dataset sizes extracted from the knowledge bases, the sizes after LLM-as-judge filtering, and the sizes after removing multi-token prediction events for each model family.

\subsection{Human Validation of LLM Filter}
\label{sec:human_ann}

Table~\ref{tab:human_annotation} reports a human validation study
of the knowledge base to final LLM-filtered dataset. For each dataset, we sample
$n=50$ examples from the knowledge base and collect independent judgments from
two annotators alongside the LLM filter's decision. We report acceptance rates from human annotators through majority vote and from the LLM filter, along with the percentage of LLM-accepted items also accepted by each annotator.

Both annotators estimate LLM precision at $0.88$ or higher on every
dataset. Therefore items retained by the filter are reliably
correct under independent human review, and the underlying judgments
appear consistent across annotators rather than annotator-specific. Moreover, the LLM is generally the stricter filter, retaining a similar or smaller fraction of examples than human annotators.

\section{Inverse Scaling Investigation}
\label{sec:inv_scaling}

\begin{table*}[t]
\centering
\small
\setlength{\tabcolsep}{10pt}
\begin{tabular}{lccccc}
\toprule
Model & $\mathbb{P}_{\theta_{C}}(\propvar_{1})$ & $\mathbb{P}_{\theta_{C}}(\propvar_{2})$ & 
$\Delta$ 
% $\mathbb{P}_{\theta_{C}}(\propvar_{2}){-}\mathbb{P}_{\theta_{C}}(\propvar_{1})$ 
& $\mathbb{P}_{\theta_{C}}(\propvar_{2}{>}\propvar_{1})$ & $\textrm{CAcc}$ \\
\midrule
\texttt{Gemma-3-270M} & \cellcolor{green!8}$0.53 (\pm \textcolor{gray}{0.03})$ & \cellcolor{green!8}$0.70 (\pm \textcolor{gray}{0.02})$ & \cellcolor{green!8}$0.17 (\pm \textcolor{gray}{0.01})$ & \cellcolor{green!8}$64.7 (\pm \textcolor{gray}{0.2})$ & \cellcolor{green!8}$75.4 (\pm \textcolor{gray}{1.2})$ \\
\texttt{Gemma-3-1B-pt} & $0.60 (\pm \textcolor{gray}{0.03})$ & $0.58 (\pm \textcolor{gray}{0.04})$ & $-0.02 (\pm \textcolor{gray}{0.00})$ & $49.5 (\pm \textcolor{gray}{0.6})$ & $57.8 (\pm \textcolor{gray}{3.3})$ \\
\texttt{Gemma-3-4B-pt} & $0.69 (\pm \textcolor{gray}{0.04})$ & $0.50 (\pm \textcolor{gray}{0.04})$ & $-0.19 (\pm \textcolor{gray}{0.00})$ & $34.7 (\pm \textcolor{gray}{0.4})$ & $44.1 (\pm \textcolor{gray}{3.4})$ \\
\texttt{Gemma-3-12B-pt} & \cellcolor{ai2lightpink}$0.70 (\pm \textcolor{gray}{0.04})$ & \cellcolor{ai2lightpink}$0.45 (\pm \textcolor{gray}{0.04})$ & \cellcolor{ai2lightpink}$-0.24 (\pm \textcolor{gray}{0.01})$ & \cellcolor{ai2lightpink}$30.1 (\pm \textcolor{gray}{0.5})$ & \cellcolor{ai2lightpink}$38.3 (\pm \textcolor{gray}{4.3})$ \\
\midrule
\texttt{Llama-3.2-1B} & \cellcolor{green!8}$0.73 (\pm \textcolor{gray}{0.03})$ & \cellcolor{green!8}$0.66 (\pm \textcolor{gray}{0.04})$ & \cellcolor{green!8}$-0.06 (\pm \textcolor{gray}{0.01})$ & \cellcolor{green!8}$45.2 (\pm \textcolor{gray}{0.8})$ & \cellcolor{green!8}$64.3 (\pm \textcolor{gray}{2.8})$ \\
\texttt{Llama-3.2-3B} & $0.67 (\pm \textcolor{gray}{0.04})$ & $0.57 (\pm \textcolor{gray}{0.04})$ & $-0.10 (\pm \textcolor{gray}{0.00})$ & $41.5 (\pm \textcolor{gray}{0.7})$ & $54.5 (\pm \textcolor{gray}{3.7})$ \\
\texttt{Llama-3.1-8B} & \cellcolor{ai2lightpink}$0.70 (\pm \textcolor{gray}{0.03})$ & \cellcolor{ai2lightpink}$0.50 (\pm \textcolor{gray}{0.04})$ & \cellcolor{ai2lightpink}$-0.20 (\pm \textcolor{gray}{0.01})$ & \cellcolor{ai2lightpink}$33.2 (\pm \textcolor{gray}{1.0})$ & \cellcolor{ai2lightpink}$43.3 (\pm \textcolor{gray}{4.1})$ \\
\midrule
\texttt{Qwen3-0.6B} & \cellcolor{green!8}$0.59 (\pm \textcolor{gray}{0.03})$ & \cellcolor{green!8}$0.76 (\pm \textcolor{gray}{0.03})$ & \cellcolor{green!8}$0.17 (\pm \textcolor{gray}{0.00})$ & \cellcolor{green!8}$67.1 (\pm \textcolor{gray}{0.4})$ & \cellcolor{green!8}$80.1 (\pm \textcolor{gray}{2.5})$ \\
\texttt{Qwen3-1.7B} & $0.64 (\pm \textcolor{gray}{0.03})$ & $0.72 (\pm \textcolor{gray}{0.03})$ & $0.08 (\pm \textcolor{gray}{0.00})$ & $60.0 (\pm \textcolor{gray}{0.3})$ & $73.1 (\pm \textcolor{gray}{3.1})$ \\
\texttt{Qwen3-4B} & $0.66 (\pm \textcolor{gray}{0.03})$ & $0.69 (\pm \textcolor{gray}{0.04})$ & $0.03 (\pm \textcolor{gray}{0.01})$ & $54.8 (\pm \textcolor{gray}{0.7})$ & $69.8 (\pm \textcolor{gray}{4.4})$ \\
\texttt{Qwen3-8B} & $0.68 (\pm \textcolor{gray}{0.03})$ & $0.67 (\pm \textcolor{gray}{0.04})$ & $-0.01 (\pm \textcolor{gray}{0.00})$ & $51.3 (\pm \textcolor{gray}{0.8})$ & $67.5 (\pm \textcolor{gray}{4.3})$ \\
\texttt{Qwen3-14B} & \cellcolor{ai2lightpink}$0.65 (\pm \textcolor{gray}{0.04})$ & \cellcolor{ai2lightpink}$0.62 (\pm \textcolor{gray}{0.05})$ & \cellcolor{ai2lightpink}$-0.03 (\pm \textcolor{gray}{0.01})$ & \cellcolor{ai2lightpink}$46.5 (\pm \textcolor{gray}{1.1})$ & \cellcolor{ai2lightpink}$59.8 (\pm \textcolor{gray}{5.2})$ \\
\bottomrule
\end{tabular}
\caption{Inverse scaling analysis for the forward implication $\propvar_{1} \implication \propvar_{2}$ across all model families. $\Delta$ is shorthand for $\mathbb{P}_{\theta_{C}}(\propvar_{2}){-}\mathbb{P}_{\theta_{C}}(\propvar_{1})$, the mean difference between conclusion and premise probabilities. $\mathbb{P}_{\theta_{C}}(\propvar_{2}{>}\propvar_{1})$ is the percentage of points where $\mathbb{P}_{\theta_{C}}(\propvar_{2})>\mathbb{P}_{\theta_{C}}(\propvar_{1})$. Calibration is done using \textit{nucleus entropy} averaged over nucelus values $p \in \{0.80, 0.85, 0.90, 0.95\}$.}
\label{tab:fi_diagnostic}
\end{table*}

%============================================================
% TRANSITIVITY
%============================================================

\begin{table*}[t]
\centering
\small
\setlength{\tabcolsep}{5pt}
\begin{tabular}{lcccccc}
\toprule
Model & $\mathbb{P}_{\theta_{C}}(\propvar_{1})$ & $\mathbb{P}_{\theta_{C}}(\propvar_{2})$ & $\mathbb{P}_{\theta_{C}}(\propvar_{3})$ & 
$\Delta$
% $\mathbb{P}_{\theta_{C}}(\propvar_{3}){-}\mathbb{P}_{\theta_{C}}(\propvar_{1})\mathbb{P}_{\theta_{C}}(\propvar_{2})$ 
& $\mathbb{P}_{\theta_{C}}(\propvar_{3}{>}\propvar_{1}\propvar_{2})$ & $\textrm{CAcc}$ \\
\midrule
\texttt{Gemma-3-270M} & \cellcolor{green!8}$0.48 (\pm \textcolor{gray}{0.03})$ & \cellcolor{green!8}$0.64 (\pm \textcolor{gray}{0.04})$ & \cellcolor{green!8}$0.64 (\pm \textcolor{gray}{0.02})$ & \cellcolor{green!8}$0.34 (\pm \textcolor{gray}{0.02})$ & \cellcolor{green!8}$82.7 (\pm \textcolor{gray}{1.8})$ & \cellcolor{green!8}$77.6 (\pm \textcolor{gray}{1.8})$ \\
\texttt{Gemma-3-1B-pt} & $0.55 (\pm \textcolor{gray}{0.03})$ & $0.61 (\pm \textcolor{gray}{0.04})$ & $0.52 (\pm \textcolor{gray}{0.03})$ & $0.20 (\pm \textcolor{gray}{0.02})$ & $68.5 (\pm \textcolor{gray}{0.7})$ & $64.6 (\pm \textcolor{gray}{0.4})$ \\
\texttt{Gemma-3-4B-pt} & $0.65 (\pm \textcolor{gray}{0.04})$ & $0.71 (\pm \textcolor{gray}{0.04})$ & $0.45 (\pm \textcolor{gray}{0.04})$ & $-0.01 (\pm \textcolor{gray}{0.01})$ & $48.5 (\pm \textcolor{gray}{0.2})$ & $42.6 (\pm \textcolor{gray}{0.3})$ \\
\texttt{Gemma-3-12B-pt} & \cellcolor{ai2lightpink}$0.64 (\pm \textcolor{gray}{0.04})$ & \cellcolor{ai2lightpink}$0.72 (\pm \textcolor{gray}{0.04})$ & \cellcolor{ai2lightpink}$0.40 (\pm \textcolor{gray}{0.04})$ & \cellcolor{ai2lightpink}$-0.05 (\pm \textcolor{gray}{0.02})$ & \cellcolor{ai2lightpink}$42.3 (\pm \textcolor{gray}{0.5})$ & \cellcolor{ai2lightpink}$36.0 (\pm \textcolor{gray}{0.5})$ \\
\midrule
\texttt{Llama-3.2-1B} & \cellcolor{green!8}$0.65 (\pm \textcolor{gray}{0.04})$ & \cellcolor{green!8}$0.84 (\pm \textcolor{gray}{0.03})$ & \cellcolor{green!8}$0.58 (\pm \textcolor{gray}{0.04})$ & \cellcolor{green!8}$0.02 (\pm \textcolor{gray}{0.01})$ & \cellcolor{green!8}$54.0 (\pm \textcolor{gray}{0.6})$ & \cellcolor{green!8}$51.3 (\pm \textcolor{gray}{2.6})$ \\
\texttt{Llama-3.2-3B} & $0.60 (\pm \textcolor{gray}{0.04})$ & $0.82 (\pm \textcolor{gray}{0.03})$ & $0.47 (\pm \textcolor{gray}{0.04})$ & $-0.02 (\pm \textcolor{gray}{0.02})$ & $47.2 (\pm \textcolor{gray}{0.4})$ & $43.5 (\pm \textcolor{gray}{1.7})$ \\
\texttt{Llama-3.1-8B} & \cellcolor{ai2lightpink}$0.62 (\pm \textcolor{gray}{0.05})$ & \cellcolor{ai2lightpink}$0.82 (\pm \textcolor{gray}{0.03})$ & \cellcolor{ai2lightpink}$0.43 (\pm \textcolor{gray}{0.04})$ & \cellcolor{ai2lightpink}$-0.08 (\pm \textcolor{gray}{0.02})$ & \cellcolor{ai2lightpink}$41.9 (\pm \textcolor{gray}{0.7})$ & \cellcolor{ai2lightpink}$33.8 (\pm \textcolor{gray}{1.4})$ \\
\midrule
\texttt{Qwen3-0.6B} & \cellcolor{green!8}$0.50 (\pm \textcolor{gray}{0.03})$ & \cellcolor{green!8}$0.81 (\pm \textcolor{gray}{0.03})$ & \cellcolor{green!8}$0.69 (\pm \textcolor{gray}{0.03})$ & \cellcolor{green!8}$0.28 (\pm \textcolor{gray}{0.01})$ & \cellcolor{green!8}$81.5 (\pm \textcolor{gray}{0.8})$ & \cellcolor{green!8}$75.7 (\pm \textcolor{gray}{0.9})$ \\
\texttt{Qwen3-1.7B} & $0.61 (\pm \textcolor{gray}{0.04})$ & $0.83 (\pm \textcolor{gray}{0.03})$ & $0.66 (\pm \textcolor{gray}{0.04})$ & $0.14 (\pm \textcolor{gray}{0.02})$ & $68.2 (\pm \textcolor{gray}{0.4})$ & $63.0 (\pm \textcolor{gray}{1.6})$ \\
\texttt{Qwen3-4B} & $0.60 (\pm \textcolor{gray}{0.04})$ & $0.81 (\pm \textcolor{gray}{0.03})$ & $0.63 (\pm \textcolor{gray}{0.04})$ & $0.14 (\pm \textcolor{gray}{0.01})$ & $65.9 (\pm \textcolor{gray}{0.3})$ & $59.6 (\pm \textcolor{gray}{3.7})$ \\
\texttt{Qwen3-8B} & \cellcolor{ai2lightpink}$0.64 (\pm \textcolor{gray}{0.05})$ & \cellcolor{ai2lightpink}$0.84 (\pm \textcolor{gray}{0.03})$ & \cellcolor{ai2lightpink}$0.61 (\pm \textcolor{gray}{0.04})$ & \cellcolor{ai2lightpink}$0.07 (\pm \textcolor{gray}{0.01})$ & \cellcolor{ai2lightpink}$59.5 (\pm \textcolor{gray}{0.1})$ & \cellcolor{ai2lightpink}$52.6 (\pm \textcolor{gray}{2.6})$ \\
\texttt{Qwen3-14B} & $0.60 (\pm \textcolor{gray}{0.04})$ & $0.81 (\pm \textcolor{gray}{0.03})$ & $0.58 (\pm \textcolor{gray}{0.04})$ & $0.09 (\pm \textcolor{gray}{0.01})$ & $60.4 (\pm \textcolor{gray}{0.7})$ & $54.9 (\pm \textcolor{gray}{2.5})$ \\
\bottomrule
\end{tabular}
\caption{Inverse scaling analysis for transitivity $(\propvar_{1} \logicand \propvar_{2}) \implication \propvar_{3}$ 
, across all model families. $\Delta$ is shorthand for $\mathbb{P}_{\theta_{C}}(\propvar_{3}){-}\mathbb{P}_{\theta_{C}}(\propvar_{1})\mathbb{P}_{\theta_{C}}(\propvar_{2})$, the mean gap between the conclusion probability and the product of premise probabilities. $\mathbb{P}_{\theta_{C}}(\propvar_{3}{>}\propvar_{1}\propvar_{2})$ is the percentage of points where $\mathbb{P}_{\theta_{C}}(\propvar_{3})>\mathbb{P}_{\theta_{C}}(\propvar_{1})\mathbb{P}_{\theta_{C}}(\propvar_{2})$. Calibration is done using \textit{nucleus entropy} averaged over nucleus values $p \in \{0.80, 0.85, 0.90, 0.95\}$.}
%\label{tab:tr_diagnostic}
\label{tab:tr_diagnostic}
\end{table*}

\begin{table*}[t]
\centering
\begin{tabular}{>{\raggedright\arraybackslash}p{3.5cm} c c cc cc cc}
\toprule
 &  &  & \multicolumn{2}{c}{$\mathbb{P}_{\theta_C}(\propvar_1)$}
 & \multicolumn{2}{c}{$\mathbb{P}_{\theta_C}(\propvar_2)$}
 & \multicolumn{2}{c}{$\rho$} \\
\cmidrule(lr){4-5}\cmidrule(lr){6-7}\cmidrule(lr){8-9}
Prefix & $a$ & $b$ & 270M & 12B & 270M & 12B & 270M & 12B \\
\midrule
Ctenocephalides Felis is a type of \dots & flea & insect & $0.07$ & $0.83$ & $0.78$ & $0.03$ & $0.94$ & $-2.21$ \\
\hline
Satsuma is a type of \dots & mandarin & fruit & $0.03$ & $0.62$ & $0.76$ & $0.12$ & $0.97$ & $-1.17$ \\
\hline
Hesperocyparis Goveniana is a type of \dots & cypress & tree & $0.07$ & $0.84$ & $0.89$ & $0.45$ & $0.97$ & $-0.85$ \\
\hline
Arbalest is a type of \dots & crossbow & weapon & $0.13$ & $0.89$ & $0.51$ & $0.55$ & $0.76$ & $-0.59$ \\
\hline
Pannier is a type of \dots & bag & container & $0.20$ & $0.83$ & $0.58$ & $0.32$ & $0.67$ & $-1.26$ \\
\bottomrule
\end{tabular}
\caption{Examples of inverse scaling for forward implication, comparing
\texttt{Gemma-3-270M} and \texttt{Gemma-3-12B} using \textit{nucleus
entropy} with nucleus value $p=0.95$. Each row defines
$\propvar_1=\mathcal{M}(\text{Prefix},a)$ and
$\propvar_2=\mathcal{M}(\text{Prefix},b)$, where $a$ and $b$ are the
premise and conclusion completions respectively.}
\label{tab:inv_scale_examples}
\end{table*}

We provide additional evidence explaining the inverse scaling phenomenon observed in the Forward Implication and Transitivity datasets in Tables \ref{tab:fi_diagnostic} and \ref{tab:tr_diagnostic}. Across all model families, the smallest model is the most consistent and the largest model is the least consistent, where consistency is measured by aggregate \textrm{CAcc} scores (\texttt{Qwen3} is a minor exception, with the second largest model being slightly more inconsistent than the largest model). In all cases the probabilities assigned to the conclusions, which are always broader consequences of the premises in \probCT, decrease with model size on average. The probabilities assigned to the more narrower premises increases with size on average for the \texttt{Gemma-3} and \texttt{Qwen3} model families, and remain similar across model sizes in the \texttt{Llama-3} family.  Specific examples exhibiting this phenomenon between \texttt{Gemma-3-270M} and \texttt{Gemma-3-12B} are shown in Table \ref{tab:inv_scale_examples}.

\section{Enforcing Factuality}
\label{sec:factuality}

\begin{table*}[t]
\centering
{\footnotesize
\setlength{\tabcolsep}{4pt}

  \begin{tabular}{lccccc}
  \toprule
  Model & Transitivity & Forward Implication & Neg.\ Consistency & Mutual Exclusivity & Spatial Exclusivity \\
  \midrule
  \texttt{Gemma-3-270M} & $40.8 (\pm \textcolor{gray}{6.9})$ & $39.5 (\pm \textcolor{gray}{5.8})$ & $7.5 (\pm \textcolor{gray}{0.4})$ & $64.1 (\pm \textcolor{gray}{0.2})$ & $32.6 (\pm \textcolor{gray}{5.3})$ \\
  \texttt{Gemma-3-1B-pt} & $39.0 (\pm \textcolor{gray}{7.2})$ & $40.0 (\pm \textcolor{gray}{6.7})$ & $10.8 (\pm \textcolor{gray}{0.9})$ & $71.1 (\pm \textcolor{gray}{2.6})$ & $28.8 (\pm \textcolor{gray}{5.0})$ \\
  \texttt{Gemma-3-4B-pt} & $47.5 (\pm \textcolor{gray}{5.6})$ & $42.9 (\pm \textcolor{gray}{6.5})$ & $17.0 (\pm \textcolor{gray}{1.4})$ & $73.9 (\pm \textcolor{gray}{4.0})$ & $41.4 (\pm \textcolor{gray}{6.9})$ \\
  \texttt{Gemma-3-12B-pt} & $45.1 (\pm \textcolor{gray}{7.9})$ & $40.7 (\pm \textcolor{gray}{6.7})$ & $13.9 (\pm \textcolor{gray}{0.9})$ & $74.9 (\pm \textcolor{gray}{4.0})$ & $48.7 (\pm \textcolor{gray}{7.9})$ \\
  \midrule
  \texttt{Llama-3.2-1B} & \cellcolor{green!8}{$64.5 (\pm \textcolor{gray}{4.3})$} & \cellcolor{green!8}{$\mathbf{60.8 (\pm \textcolor{gray}{5.3})}$} & $15.9 (\pm \textcolor{gray}{1.3})$ & \cellcolor{green!8}{$79.7 (\pm \textcolor{gray}{3.6})$} & \cellcolor{green!8}{$\mathbf{52.9 (\pm \textcolor{gray}{6.4})}$} \\
  \texttt{Llama-3.2-3B} & $53.4 (\pm \textcolor{gray}{7.9})$ & $45.9 (\pm \textcolor{gray}{8.0})$ & $13.6 (\pm \textcolor{gray}{1.8})$ & $79.0 (\pm \textcolor{gray}{4.0})$ & $49.3 (\pm \textcolor{gray}{6.6})$ \\
  \texttt{Llama-3.1-8B} & \cellcolor{ai2lightpink} $51.3 (\pm \textcolor{gray}{7.2})$ & \cellcolor{ai2lightpink} $43.9 (\pm \textcolor{gray}{7.8})$ & $\mathbf{20.2 (\pm \textcolor{gray}{2.1})}$ & \cellcolor{ai2lightpink} $78.8 (\pm \textcolor{gray}{3.4})$ & \cellcolor{ai2lightpink} $38.8 (\pm \textcolor{gray}{7.1})$ \\
  \midrule
  \texttt{Qwen3-0.6B} & $58.9 (\pm \textcolor{gray}{5.8})$ & $51.6 (\pm \textcolor{gray}{5.8})$ & $11.3 (\pm \textcolor{gray}{0.7})$ & $73.9 (\pm \textcolor{gray}{0.5})$ & $21.9 (\pm \textcolor{gray}{5.2})$ \\
  \texttt{Qwen3-1.7B} & $65.5 (\pm \textcolor{gray}{5.1})$ & $56.7 (\pm \textcolor{gray}{4.9})$ & \cellcolor{green!8}{$16.3 (\pm \textcolor{gray}{3.1})$} & $77.2 (\pm \textcolor{gray}{2.6})$ & $44.4 (\pm \textcolor{gray}{5.2})$ \\
  \texttt{Qwen3-4B} & $61.3 (\pm \textcolor{gray}{5.3})$ & $54.3 (\pm \textcolor{gray}{4.5})$ & \cellcolor{ai2lightpink} $9.5 (\pm \textcolor{gray}{1.2})$ & $81.9 (\pm \textcolor{gray}{2.2})$ & $35.1 (\pm \textcolor{gray}{7.0})$ \\
  \texttt{Qwen3-8B} & $\mathbf{67.4 (\pm \textcolor{gray}{5.3})}$ & $57.9 (\pm \textcolor{gray}{6.5})$ & $12.4 (\pm \textcolor{gray}{1.7})$ & $\mathbf{82.8 (\pm \textcolor{gray}{2.7})}$ & $33.9 (\pm \textcolor{gray}{5.9})$ \\
  \texttt{Qwen3-14B} & $59.6 (\pm \textcolor{gray}{5.3})$ & $53.3 (\pm \textcolor{gray}{7.9})$ & $13.4 (\pm \textcolor{gray}{2.2})$ & $81.0 (\pm \textcolor{gray}{3.7})$ & $35.3 (\pm \textcolor{gray}{7.4})$ \\
  \bottomrule
  \end{tabular}
}
\caption{Enforcing Factuality: CAcc (\%) over conjunction of true statements using \textit{nucleus entropy} calibration averaged over $p \in \{0.8, 0.85, 0.9, 0.95\}$.}
\label{tab:wmc_joint_discrete}
\end{table*}

Our main results utilize general logical constraints without imposing factuality. However, \probCT is generated from knowledge bases and therefore comes with ground truth factual labels. This allows us to evaluate each dataset under a modified logical constraint that consists of the conjunction of all true statements in a data point. These results are presented in Figure \ref{tab:wmc_joint_discrete}.

\section{Additional Calibration Methods}
\label{sec:other_calibration}

Here we consider alternative choices for the calibration method; that is we consider different choices of the value $C$ in Equation \ref{eq:calibration}. As a reminder, our default method uses the entropy of the top-$p$ nucleus of the local distribution, and results are reproduced in Table \ref{tab:wmc_main_agg_repeat} for ease of comparison. \textbf{Entropy calibration}, shown in Table \ref{tab:wmc_cal_entropy}, sets $C$ to be the entropy of the full next-token distribution, dropping the nucleus restriction. \textbf{Nucleus calibration}, shown in Table \ref{tab:wmc_cal_nucleus}, sets $C$ to be the log count of tokens in the top-$p$ nucleus, replacing entropy with a uniform count over the same set. 

Both alternatives generally preserve the qualitative findings of Table \ref{tab:wmc_main_agg}: models generally perform poorly, inverse scaling on Transitivity and Forward Implication is present across model families, and the task difficulty ordering is unchanged. Spatial Exclusivity is the most sensitive to calibration method, with meaningfully improved scores here compared to nucleus entropy calibration across model families and sizes. For example, \texttt{Gemma-3-12B} has an average score of $18.5$ for nucleus entropy calibration, while it has scores of $59.4$ and $44.4$ for nucleus calibration and entropy calibration respectively. Scores on Negation Consistency remain similarly poor across model families and calibration methods. 

\begin{table*}[t]
\centering
{\footnotesize
\setlength{\tabcolsep}{4pt}

  \begin{tabular}{lccccc}
  \toprule
  Model & Transitivity & Forward Implication & Neg.\ Consistency & Mutual Exclusivity & Spatial
  Exclusivity \\
  \midrule
  \texttt{Gemma-3-270M} & \cellcolor{green!10}{$\mathbf{81.0} (\pm \textcolor{gray}{2.3})$} &
  \cellcolor{green!8}{$79.7 (\pm \textcolor{gray}{0.5})$} & $18.5 (\pm \textcolor{gray}{1.4})$ &
  $48.7 (\pm \textcolor{gray}{0.8})$ & $18.5 (\pm \textcolor{gray}{4.8})$ \\
  \texttt{Gemma-3-1B-pt} & $69.0 (\pm \textcolor{gray}{1.5})$ & $65.1 (\pm \textcolor{gray}{1.5})$ & $15.9
  (\pm \textcolor{gray}{0.8})$ & $57.3 (\pm \textcolor{gray}{3.9})$ & $12.4 (\pm
  \textcolor{gray}{4.8})$ \\
  \texttt{Gemma-3-4B-pt} & $46.9 (\pm \textcolor{gray}{1.1})$ & $50.8 (\pm \textcolor{gray}{1.7})$ & $18.2
  (\pm \textcolor{gray}{0.9})$ & $56.8 (\pm \textcolor{gray}{5.2})$ & $19.9 (\pm
  \textcolor{gray}{6.6})$ \\
  \texttt{Gemma-3-12B-pt} & \cellcolor{ai2lightpink} $40.4 (\pm \textcolor{gray}{1.6})$ & \cellcolor{ai2lightpink} $45.2 (\pm \textcolor{gray}{2.0})$ & $19.2
  (\pm \textcolor{gray}{0.9})$ & $54.7 (\pm \textcolor{gray}{5.4})$ & $23.8 (\pm
  \textcolor{gray}{7.7})$ \\
  \midrule
  \texttt{Llama-3.2-1B} & \cellcolor{green!8}{$52.7 (\pm \textcolor{gray}{2.4})$} &
  \cellcolor{green!8}{$68.7 (\pm \textcolor{gray}{1.8})$} & \cellcolor{green!8}{$\mathbf{23.1}
  (\pm \textcolor{gray}{2.1})$} & \cellcolor{green!8}{$64.0 (\pm \textcolor{gray}{3.8})$} &
  \cellcolor{green!8}{$\mathbf{31.5} (\pm \textcolor{gray}{7.4})$} \\
  \texttt{Llama-3.2-3B} & $45.5 (\pm \textcolor{gray}{1.3})$ & $60.8 (\pm \textcolor{gray}{2.1})$ & \cellcolor{ai2lightpink} $21.6
  (\pm \textcolor{gray}{2.6})$ & $61.6 (\pm \textcolor{gray}{6.0})$ & $25.3 (\pm
  \textcolor{gray}{6.9})$ \\
  \texttt{Llama-3.1-8B} & \cellcolor{ai2lightpink} $35.2 (\pm \textcolor{gray}{1.5})$ & \cellcolor{ai2lightpink} $50.0 (\pm \textcolor{gray}{1.2})$ & $22.7
  (\pm \textcolor{gray}{1.3})$ & \cellcolor{ai2lightpink} $60.0 (\pm \textcolor{gray}{5.9})$ & \cellcolor{ai2lightpink} $16.1 (\pm
  \textcolor{gray}{6.4})$ \\
  \midrule
  \texttt{Qwen3-0.6B} & \cellcolor{green!8}{$76.4 (\pm \textcolor{gray}{0.5})$} &
  \cellcolor{green!8}{$\mathbf{82.8} (\pm \textcolor{gray}{1.5})$} & $17.2 (\pm
  \textcolor{gray}{0.6})$ & $58.2 (\pm \textcolor{gray}{0.7})$ & $9.1 (\pm \textcolor{gray}{3.2})$
  \\
  \texttt{Qwen3-1.7B} & $64.2 (\pm \textcolor{gray}{1.2})$ & $75.3 (\pm \textcolor{gray}{1.6})$ & $18.9 (\pm
  \textcolor{gray}{3.1})$ & $62.9 (\pm \textcolor{gray}{2.1})$ & $25.9 (\pm \textcolor{gray}{7.2})$
  \\
  \texttt{Qwen3-4B} & $60.8 (\pm \textcolor{gray}{3.2})$ & $73.0 (\pm \textcolor{gray}{2.6})$ & $19.3 (\pm
  \textcolor{gray}{0.6})$ & $67.7 (\pm \textcolor{gray}{5.0})$ & $16.6 (\pm \textcolor{gray}{5.1})$
  \\
  \texttt{Qwen3-8B} & \cellcolor{ai2lightpink} $53.7 (\pm \textcolor{gray}{2.2})$ & $69.6 (\pm \textcolor{gray}{2.7})$ & $20.8 (\pm
  \textcolor{gray}{0.1})$ & $\mathbf{68.2} (\pm \textcolor{gray}{4.3})$ & $14.8 (\pm
  \textcolor{gray}{5.3})$ \\
  \texttt{Qwen3-14B} & $56.4 (\pm \textcolor{gray}{2.1})$ & \cellcolor{ai2lightpink} $64.1 (\pm \textcolor{gray}{3.6})$ & $20.9 (\pm
  \textcolor{gray}{1.2})$ & $65.9 (\pm \textcolor{gray}{5.6})$ & $12.6 (\pm \textcolor{gray}{5.4})$
  \\
  \bottomrule
  \end{tabular}
}
\caption{CAcc (\%) using \textit{nucleus entropy} calibration averaged over
nucleus values $p \in \{0.8, 0.85, 0.9, 0.95\}$. Reproduced from Table~\ref{tab:wmc_main_agg} for ease of comparison with the alternative calibrations in Tables~\ref{tab:wmc_cal_entropy} and~\ref{tab:wmc_cal_nucleus}}
\label{tab:wmc_main_agg_repeat}
\end{table*}

\begin{table*}[t]
\centering
{\footnotesize
\setlength{\tabcolsep}{4pt}
  \begin{tabular}{lccccc}
  \toprule
  Model & Transitivity & Forward Implication & Neg.\ Consistency & Mutual Exclusivity & Spatial Exclusivity \\
  \midrule
  \texttt{Gemma-3-270M}  & \cellcolor{green!8}{$79.2$} & \cellcolor{green!8}{$81.5$} & \cellcolor{green!8}{$\mathbf{20.9}$} & $47.5$ & $31.4$ \\
  \texttt{Gemma-3-1B-pt} & $65.2$ & $68.9$ & $18.3$ & $65.5$ & $26.5$ \\
  \texttt{Gemma-3-4B-pt} & $48.1$ & $56.8$ & \cellcolor{ai2lightpink} $17.9$ & $68.8$ & $37.0$ \\
  \texttt{Gemma-3-12B-pt} & \cellcolor{ai2lightpink} $40.5$ & \cellcolor{ai2lightpink} $52.1$ & $18.3$ & $72.0$ & $44.4$ \\
  \midrule
  \texttt{Llama-3.2-1B} & \cellcolor{green!8}{$58.3$} & \cellcolor{green!8}{$75.6$} & $17.6$ & $72.6$ & \cellcolor{green!8}{$\mathbf{50.9}$} \\
  \texttt{Llama-3.2-3B} & $50.8$ & $66.1$ & $16.1$ & $74.9$ & $45.5$ \\
  \texttt{Llama-3.1-8B} & \cellcolor{ai2lightpink} $43.0$ & \cellcolor{ai2lightpink} $56.6$ & $18.3$ & $75.0$ & \cellcolor{ai2lightpink} $34.2$ \\
  \midrule
  \texttt{Qwen3-0.6B} & \cellcolor{green!8}{$\mathbf{79.8}$} & \cellcolor{green!8}{$\mathbf{85.9}$} & $15.5$ & $56.6$ & $19.4$ \\
  \texttt{Qwen3-1.7B} & $69.7$ & $79.7$ & $13.7$ & $68.5$ & $41.9$ \\
  \texttt{Qwen3-4B} & $68.1$ & $78.0$ & $17.2$ & $76.1$ & $32.1$ \\
  \texttt{Qwen3-8B} & $62.5$ & $75.6$ & $19.2$ & $\mathbf{78.9}$ & $30.2$ \\
  \texttt{Qwen3-14B} & \cellcolor{ai2lightpink} $61.2$ & \cellcolor{ai2lightpink} $70.5$ & $18.7$ & $78.0$ & $30.7$ \\
  \bottomrule
  \end{tabular}
}
\caption{CAcc (\%) using \textit{entropy} calibration.}
\label{tab:wmc_cal_entropy}
\end{table*}

\begin{table*}[t]
\centering
{\footnotesize
\setlength{\tabcolsep}{4pt}
  \begin{tabular}{lccccc}
  \toprule
  Model & Transitivity & Forward Implication & Neg.\ Consistency & Mutual Exclusivity & Spatial Exclusivity \\
  \midrule
  \texttt{Gemma-3-270M}  & \cellcolor{green!8}{$77.8 (\pm \textcolor{gray}{1.5})$} & \cellcolor{green!8}{$83.7 (\pm \textcolor{gray}{2.2})$} & $17.0 (\pm \textcolor{gray}{3.3})$ & $42.2 (\pm \textcolor{gray}{5.7})$ & $55.3 (\pm \textcolor{gray}{13.1})$ \\
  \texttt{Gemma-3-1B-pt} & $68.0 (\pm \textcolor{gray}{3.6})$ & $74.2 (\pm \textcolor{gray}{5.8})$ & \cellcolor{green!8}{$\mathbf{17.4 (\pm \textcolor{gray}{2.4})}$} & $63.4 (\pm \textcolor{gray}{1.4})$ & $42.9 (\pm \textcolor{gray}{15.0})$ \\
  \texttt{Gemma-3-4B-pt} & $53.2 (\pm \textcolor{gray}{7.4})$ & $63.7 (\pm \textcolor{gray}{7.7})$ & \cellcolor{ai2lightpink} $15.4 (\pm \textcolor{gray}{3.5})$ & $72.2 (\pm \textcolor{gray}{4.7})$ & $52.9 (\pm \textcolor{gray}{16.5})$ \\
  \texttt{Gemma-3-12B-pt} & \cellcolor{ai2lightpink} $45.4 (\pm \textcolor{gray}{7.7})$ & \cellcolor{ai2lightpink} $56.5 (\pm \textcolor{gray}{8.2})$ & $17.2 (\pm \textcolor{gray}{2.3})$ & $71.8 (\pm \textcolor{gray}{6.2})$ & $59.4 (\pm \textcolor{gray}{17.1})$ \\
  \midrule
  \texttt{Llama-3.2-1B} & \cellcolor{green!8}{$68.6 (\pm \textcolor{gray}{8.5})$} & \cellcolor{green!8}{$82.0 (\pm \textcolor{gray}{5.6})$} & $11.6 (\pm \textcolor{gray}{4.9})$ & $70.1 (\pm \textcolor{gray}{1.5})$ & \cellcolor{green!8}{$\mathbf{68.6 (\pm \textcolor{gray}{12.9})}$} \\
  \texttt{Llama-3.2-3B} & $56.0 (\pm \textcolor{gray}{7.8})$ & $73.6 (\pm \textcolor{gray}{6.7})$ & $12.3 (\pm \textcolor{gray}{3.9})$ & $76.4 (\pm \textcolor{gray}{4.3})$ & $59.5 (\pm \textcolor{gray}{14.7})$ \\
  \texttt{Llama-3.1-8B} & \cellcolor{ai2lightpink} $50.5 (\pm \textcolor{gray}{9.6})$ & \cellcolor{ai2lightpink} $64.4 (\pm \textcolor{gray}{10.0})$ & $13.4 (\pm \textcolor{gray}{5.1})$ & $75.1 (\pm \textcolor{gray}{6.4})$ & \cellcolor{ai2lightpink} $47.3 (\pm \textcolor{gray}{17.7})$ \\
  \midrule
  \texttt{Qwen3-0.6B} & \cellcolor{green!8}{$\mathbf{82.2 (\pm \textcolor{gray}{3.7})}$} & \cellcolor{green!8}{$\mathbf{90.0 (\pm \textcolor{gray}{3.4})}$} & $9.8 (\pm \textcolor{gray}{3.3})$ & $48.3 (\pm \textcolor{gray}{8.0})$ & $46.0 (\pm \textcolor{gray}{14.7})$ \\
  \texttt{Qwen3-1.7B} & $77.7 (\pm \textcolor{gray}{6.2})$ & $85.5 (\pm \textcolor{gray}{4.0})$ & $8.5 (\pm \textcolor{gray}{3.2})$ & $62.4 (\pm \textcolor{gray}{4.5})$ & $61.9 (\pm \textcolor{gray}{11.8})$ \\
  \texttt{Qwen3-4B} & $74.4 (\pm \textcolor{gray}{6.3})$ & $83.9 (\pm \textcolor{gray}{4.2})$ & $11.6 (\pm \textcolor{gray}{4.5})$ & $71.2 (\pm \textcolor{gray}{4.1})$ & $55.2 (\pm \textcolor{gray}{17.2})$ \\
  \texttt{Qwen3-8B} & $70.8 (\pm \textcolor{gray}{8.0})$ & $82.0 (\pm \textcolor{gray}{5.5})$ & $12.7 (\pm \textcolor{gray}{4.4})$ & $75.7 (\pm \textcolor{gray}{1.2})$ & $48.0 (\pm \textcolor{gray}{16.0})$ \\
  \texttt{Qwen3-14B} & \cellcolor{ai2lightpink} $69.9 (\pm \textcolor{gray}{7.2})$ & \cellcolor{ai2lightpink} $79.5 (\pm \textcolor{gray}{6.5})$ & $13.2 (\pm \textcolor{gray}{4.7})$ & $\mathbf{78.7 (\pm \textcolor{gray}{3.4})}$ & $46.6 (\pm \textcolor{gray}{18.9})$ \\
  \bottomrule
  \end{tabular}
}
\caption{CAcc (\%) using \textit{nucleus} calibration averaged over nucleus values
$p \in \{0.8, 0.85, 0.9, 0.95\}$.}
\label{tab:wmc_cal_nucleus}
\end{table*}

\clearpage
\twocolumn
\section{Additional Probabilistic Consistency Histograms}
\label{sec:histograms}

\vspace{1em}
\begin{minipage}{\textwidth}
\vspace*{\fill}
\centering

\wmcplotrow{Gemma-3-270M}{cal_nucleus_entropy095}
\wmcplotrow{Gemma-3-1B-pt}{cal_nucleus_entropy095}
\wmcplotrow{Gemma-3-4B-pt}{cal_nucleus_entropy095}
\wmcplotrowlabeled{Gemma-3-12B-pt}{cal_nucleus_entropy095}

\captionof{figure}{WMC score distributions for \texttt{Gemma-3} base models  using \textit{nucleus entropy} calibration with nucleus value $p=0.95$. Points highlighted in red are inconsistent, while points highlighted in blue are consistent. The solid vertical black line corresponds to the average probabilistic consistency. Each figure is labeled with the corresponding aggregate \textrm{CAcc} score.}
\label{fig:gemma_pretrained_constraints}
\vspace*{\fill}
\end{minipage}

\vspace{3em}
\begin{minipage}{\textwidth}
\centering

\wmcplotrow{Llama-3.2-1B}{cal_nucleus_entropy095}
\wmcplotrow{Llama-3.2-3B}{cal_nucleus_entropy095}
\wmcplotrowlabeled{Llama-3.1-8B}{cal_nucleus_entropy095}

\captionof{figure}{Probabilistic Consistency score distributions for \texttt{Llama-3} pretrained models using \textit{nucleus entropy} calibration with nucleus value $p=0.95$. Points highlighted in red are inconsistent, while points highlighted in blue are consistent. The solid vertical black line corresponds to the average probabilistic consistency. Each figure is labeled with the corresponding aggregate \textrm{CAcc} score.}
\label{fig:llama_pretrained_constraints}
\end{minipage}

\clearpage
\vspace{1em}
\begin{minipage}{\textwidth}
\centering

\wmcplotrow{Qwen3-0.6B}{cal_nucleus_entropy095}
\wmcplotrow{Qwen3-1.7B}{cal_nucleus_entropy095}
\wmcplotrow{Qwen3-4B}{cal_nucleus_entropy095}
\wmcplotrow{Qwen3-8B}{cal_nucleus_entropy095}
\wmcplotrowlabeled{Qwen3-14B}{cal_nucleus_entropy095}

\captionof{figure}{Probabilistic Consistency score distributions for \texttt{Qwen3} pretrained models using \textit{nucleus entropy} calibration with nucleus value $p=0.95$. Points highlighted in red are inconsistent, while points highlighted in blue are consistent. The solid vertical black line corresponds to the average probabilistic consistency. Each figure is labeled with the corresponding aggregate \textrm{CAcc} score.}
\label{fig:qwen_base_constraints}
\end{minipage}

\end{document}